\documentclass[10pt,twoside]{article}
\usepackage{amsmath}
\usepackage[utf8]{inputenc}
\usepackage{url}
\usepackage{dsfont}
\usepackage{graphicx}
\usepackage{yhmath}
\usepackage[table]{xcolor}
\usepackage{mathrsfs} 
\usepackage{amssymb}
\usepackage{url}
\usepackage{makecell}
\usepackage{arydshln}
\usepackage{multirow}
\usepackage{arydshln}
\usepackage{mathtools}
\usepackage{stmaryrd}  
 \usepackage{amsthm,amssymb}
 \usepackage{subcaption}
 \usepackage{lipsum}
 \usepackage{nccmath}

\input epsf
\def\limiten{\renewcommand{\arraystretch}{0.5}
\begin{array}[t]{c}\stackrel{}{\longrightarrow} \\
{\scriptstyle n\rightarrow
\infty}\end{array}\renewcommand{\arraystretch}{1}}

\numberwithin{equation}{section}

\newtheorem{thm}{Theorem}[section]

\newtheorem{Corol}[thm]{Corollary}
\newtheorem{Def}[thm]{Definition}

\newtheorem{lem}[thm]{Lemma}

\newtheorem{prop}[thm]{Proposition}

\newcommand{\E}{\ensuremath{\mathbb{E}}}
\newcommand{\R}{\ensuremath{\mathbb{R}}}
\newcommand{\Z}{\ensuremath{\mathbb{Z}}}

\newcommand{\N}{\ensuremath{\mathbb{N}}}

\newcommand{\var}{\ensuremath{\mathrm{Var}}}

\definecolor{grisclair}{gray}{0.9}
\font\dsrom=dsrom10 scaled 1200
\def \ind{\textrm{\dsrom{1}}}
\DeclareMathOperator*{\argmin}{argmin}

\newcommand{\mk}{ { \mathcal{K}} }
\newcommand{\mx}{ { \mathcal{X}} }
\newcommand{\my}{ { \mathcal{Y}} }

\renewcommand{\arraystretch}{.8}

\renewcommand{\quote}[1]{``#1''}

\begin{document}

\title{\bf Deep regression learning from dependent observations with minimum error entropy principle}
 \maketitle \vspace{-1.0cm}

\begin{center}
      William Kengne
   and 
     Modou Wade
 \end{center}

  \begin{center}
  { \it 
  Institut Camille Jordan, Université Jean Monnet, 23 Rue Dr Paul Michelon 42023 Saint-Etienne Cedex 2, France\\  
 THEMA, CY Cergy Paris Université, 33 Boulevard du Port, 95011 Cergy-Pontoise Cedex, France\\
  E-mail:   william.kengne@univ-st-etienne.fr  ; modou.wade@cyu.fr\\
  }
\end{center}
 \pagestyle{myheadings}

\markboth{Deep regression learning with minimum error entropy}{Kengne and Wade}

\medskip

\textbf{Abstract}:
This paper considers nonparametric regression from strongly mixing observations.
The proposed approach is based on deep neural networks with minimum error entropy (MEE) principle. 
We study two estimators: the non-penalized deep neural network (NPDNN) and the sparse-penalized deep neural network (SPDNN) predictors.
Upper bounds of the expected excess risk are established for both estimators over the classes of H\"older and composition H\"older functions.
For the models with Gaussian error, the rates of the upper bound obtained match (up to a logarithmic factor) with the lower bounds established in \cite{schmidt2020nonparametric}, showing that both the MEE-based NPDNN and SPDNN estimators from strongly mixing data can achieve the minimax optimal convergence rate. 

\medskip
 

\medskip

\section{Introduction}
Deep learning has achieved impressive results in a wide range of learning tasks, such as image processing \cite{krizhevsky2017imagenet}, speech recognition \cite{hinton2012deep}, and, more generally, in various fields of artificial intelligence (AI).
However, theoretical support and guarantees of deep neural networks (DNN) algorithms in many frameworks are still a challenge nowadays.
In recent years, numerous studies have contributed to understand the theoretical properties of DNN estimators.
See, for instance, \cite{schmidt2020nonparametric}, \cite{imaizumi2022advantage} \cite{jiao2023deep}, \cite{fan2024noise}, \cite{zhang2024classification}, \cite{shen2024nonparametric},  (and the references therein), for some results from independent and identical distributed (i.i.d) observations.
For some recent results from dependent data, see, among other works, \cite{ma2022theoretical},\cite{kohler2023rate}, \cite{kengne2025robust}, \cite{kurisu2025adaptive}, \cite{kengne2025deep}, \cite{alquier2025minimax}.
For nonparametric regression, most of the results established on the properties of DNN estimators are based on the $L_2$ (least squares) loss, see for example, \cite{schmidt2020nonparametric}, \cite{kohler2021rate}, \cite{ma2022theoretical}, \cite{imaizumi2022advantage}, \cite{jiao2023deep}, \cite{kohler2023rate}, \cite{kohler2023rate}, \cite{kurisu2025adaptive}.  
Such estimator, which is perfect for problems
involving Gaussian noise is sensitive to non-Gaussian and heavy tailed errors, see, for instance \cite{erdogmus2000comparison}, \cite{hu2013learning}, \cite{chen2018kernel}. 
There are several works that consider an entropy criteria for nonparametric regression, see for example, \cite{hu2015regularization}, \cite{hu2013learning}, \cite{fan2016consistency}, \cite{guo2020distributed}, \cite{hu2021kernel}.  
But, theoretical properties of entropy-based DNN estimators for nonparametric regression have not yet been much studied in the literature.
We refer to \cite{wang2025deep} and \cite{chen2026maximum} for some recent works on this direction.

\medskip

We consider a stationary and ergodic process $\{Z_t =(X_t, Y_t), t \in \Z \} $, from the nonparametric regression model:
\begin{equation}\label{equa_model_reg}
Y_t = h_0(X_t) + \xi_t, ~ X_t \sim P_{X_t}
\end{equation}
where $h_0 : \R ^d \to \R$ is the unknown regression function, $(\xi_t)_{t \in \Z}$ is a centered i.i.d. error process and $\xi_t$ is independent of the input variable $X_t$.
The input and the output spaces are $\mathcal{X} \subset \R^d $ (with $d \in \N$) and $\mathcal{Y}  \subset \R$ respectively.
It is assumed that $\xi_0$ admits a density $f$ with respect to the Lebesgue measure.
Let $ D_n \coloneqq  \{ (X_1, Y_1), \dots, (X_n, Y_n) \}$ be a training sample generated from the process $\{Z_t =(X_t, Y_t), t \in \Z \} $.  
We deal with a class of DNN estimators $\mathcal{H}_{\sigma}(L_n, N_n, B_n, F_n, S_n)$ (defined in Section \ref{Def_DNNs}, see (\ref{DNNs_Constraint})).
Also, let $\mathcal{F}$ denote the class of all measurable functions from $\mathcal{X}$ to $\mathcal{Y}$. 
The minimum error entropy (MEE) principle aims to choose the predictor $h \in \mathcal{H}_{\sigma}(L_n, N_n, B_n, F_n, S_n)$ that minimizes an entropy of the error $Y_0 - h(X_0)$.
In this work, we deal with Shannon's entropy, and consider the risk for all $h \in \mathcal{F}$,
\begin{equation}\label{equa_def_risk}
 R(h) = \E_{Z_0} [-\log f(Y_0 - h(X_0))], ~ \text{with}~ Z_0 = (X_0, Y_0).  
\end{equation}
So, the corresponding loss function is,
\begin{equation}\label{loss_entropy}
\ell(h(X_0), Y_0) = -\log f(Y_0 - h(X_0)).
\end{equation}
The target predictor (assumed to exist) $h^{*} \in \mathcal{F} $ satisfies:
\begin{equation}\label{best_pred_F}
R(h^{*}) =  \underset{h \in \mathcal{F}}{\inf}R(h).
\end{equation}
For any $h \in \mathcal{F}$, the excess risk is given by:
\begin{equation}\label{equa_excess_risk}
\mathcal{E}_{Z_0}(h) = R(h) - R(h^{*}), ~ \text{with} ~ Z_0 = (X_0, Y_0).
\end{equation}

\medskip

The density $f$ of the error is assumed to be known (see discussion below, Section \ref{discussion}).
We aim to build from $D_n$, a DNN predictor $\widehat{h}_n \in \mathcal{H}_{\sigma}(L_n, N_n, B_n, F_n, S_n)$ with suitably chosen network architecture $(L_n, N_n, B_n, F_n, S_n)$,  that minimizes the empirical version of the entropy (risk). 
Define the MEE-based non-penalized DNN (NPDNN) estimator by:
\begin{equation}\label{Non_Pen_DNNs_Estimators}
\widehat{h}_{n, NP} = \underset{h \in  \mathcal{H}_{\sigma}(L_n, N_n, B_n, F_n, S_n)}{\argmin} \Big[ -\dfrac{1}{n} \sum_{i=1}^{n} \Big(\log f(Y_i - h(X_i)) \Big) \Big].
\end{equation}
For this estimator $\widehat{h}_{n, NP}$, the regularization is performed by the sparsity parameter $S_n$. 
In the following, we also consider the sparse-penalized regularization with the DNN class $\mathcal{H}_{\sigma}(L_n, N_n, B_n, F_n)$ (defined in Section \ref{Def_DNNs}, see (\ref{DNNs_no_Constraint})).
So, the MEE-based sparse-penalized DNN (SPDNN) estimator $\widehat{h}_{n,SP}$ is given by: 
\begin{equation}\label{Pen_DNNs_Estimators_v1}
\widehat{h}_{n,SP} = \underset{h \in  \mathcal{H}_{\sigma}(L_n, N_n, B_n, F)}{\argmin} \Big[ - \dfrac{1}{n} \sum_{i=1}^{n} \Big(\log f(Y_i - h(X_i)) \Big) + J_n(h) \Big],
\end{equation}
where $J_{n}(h) $ denotes the sparse penalty term, defined by: 
\begin{equation}\label{equa_penalty_term_v1}
J_n(h) = \sum_{j=1}^{p} \pi_{\lambda_n, \tau_n} \big( |\theta_j(h)| \big),
\end{equation}
where $\pi_{\lambda_n, \tau_n} : [0, \infty) \to [0, \infty)$ is a function with two tuning parameters $\lambda_n, \tau_n > 0$, and $\theta_j(h)$ is the $j$-th component of $\theta(h) = \big(\theta_1(h), \cdots, \theta_p(h) \big)^T$, which is the parameter of the network estimator $h \in \mathcal{H}_{\sigma}(L_n, N_n, B_n, F)$, and $^T$ denotes the transpose.
 We assume that $\pi_{\lambda_n, \tau_n}$ satisfies the two following conditions (see also \cite{kurisu2025adaptive}):
\begin{itemize}
\item [(i)] $\pi_{\lambda_n, \tau_n}(0) = 0$ and $\pi_{\lambda_n, \tau_n}(\cdot)$ is non-decreasing.
\item[(ii)] $\pi_{\lambda_n, \tau_n}(x) = \lambda_n$ if $x > \tau_n$.   
\end{itemize}
Examples of such penalty $J_n(h)$ are the clipped $L_1$ penalty considered in \cite{zhang2010analysis}, \cite{ohn2022nonconvex}, defined by for all $x \ge 0$ as:
\begin{equation}\label{equa_clipped_L1_penalty}
\pi_{\lambda_n, \tau_n}(x) = \lambda_n \Big(\dfrac{x}{\tau_n} \land 1 \Big), 
\end{equation}
the SCAD penalty considered by \cite{fan2001variable}, the minimax concave penalty \cite{zhang2010nearly} or the seamless $L_0$ penalty in \cite{dicker2013variable}, see \cite{kurisu2025adaptive}.
We will be interested in establishing a bound of the excess risk $\mathcal{E}_{Z_0}(\widehat{h}_n)$, and studying how fast it tends to zero. 

\medskip

As pointed out above, most of the results on the theoretical properties of the DNN predictor established in the literature are based on the $L_2$ loss.
Such method minimizes the variance of the error variable $Y_0-h(X_0)$ (for a predictor $h$), so it takes into consideration only the first two moments of the error variable.
Such estimators cannot work very well for problems involving heavy tailed or non-Gaussian error, and are not robust to outliers, see also \cite{hu2013learning}. 
To solve the robustness question in nonparametric regression from DNN, several authors have considered the Huber loss, for example \cite{fan2024noise}, \cite{kengne2025robust}, \cite{kengne2025deep}. Except for the tuning parameter which needs to be calibrated from data, the Huber loss combines the advantages of the square and the absolute losses, and enjoys the minimax optimal property.

\medskip

 In this contribution, we deal with the Shannon entropy for nonparametric regression from strongly mixing observations.
 Two estimators, the MEE-based NPDNN and the SPDNN predictors are proposed. For each of these estimators, the upper bounds of the expected excess risk are established over the classes of H\"older and composition H\"older functions. 
When the error is Gaussian, these estimators are optimal (up to a logarithmic factor) in the minimax sense. 
 
\medskip 

 Unlike the least square method which takes into consideration only the first two moments of the error, the MEE criteria considered take into account the moments of all orders of the error variable in the entropy.
 Thus, MEE-based estimators proposed have the robustness to deal with non-Gaussian models or models with heavy-tailed noise. We also refer to \cite{erdogmus2003convergence} for some robustness properties of the MEE type estimators.
In addition, let us emphasize that the corresponding loss function from Shannon's entropy (see (\ref{loss_entropy})) is not Lipschitz continuous for example if the error is Gaussian. 
Therefore, some recent developments on DNN theory based on Lipschitz continuous loss (see for instance \cite{kengne2025general}), do not apply to such entropy criteria.

\medskip

In this work, we assume that the density of the noise is known. This could be a limitation for practical applications.
An interesting extension of this work is to assume that the density of the error is unknown, and estimate it, for example, from the kernel method, before using it in the calculation of the empirical version of entropy.
But for the Shannon's entropy considered, this issue is complex and remains a challenge, see the discussion in Section \ref{discussion}.

\medskip
The rest of the paper is organized as follows.
In Section \ref{asump}, we set some notations, assumptions and present the DNN class considered. 
 We derived an excess risk bound of the MEE-based NPDNN estimator in Section \ref{excess_NPDNN}. 
 For the MEE-based SPDNN predictor, an oracle inequality and an  excess risk bound are established in Section \ref{excess_risk_SPDNN}.
Section \ref{example} considers the example of the class of Subbotin distributions whereas Section \ref{discussion} provides some perspectives and discusses possible extensions of this work.
Section \ref{prove} is devoted to the proofs of the main results.

\section{Notations, assumptions and deep neural networks}\label{asump}
\subsection{Notations and assumptions}
Let $d \in \N$, $E_1$ and $E_2$ be subsets of separable Banach spaces equipped with norms $\| \cdot\|_{E_1}$ and $\| \cdot\|_{E_2}$ respectively.
Let us set some notations. 
\begin{itemize}
%
\item  For any $h: E_1 \rightarrow E_2$ and $\epsilon >0$, $B(h, \epsilon) $ is the ball of radius $\epsilon$ centered at $h$, that is,
\[ B (h, \epsilon) = \big\{ f: E_1 \rightarrow E_2, ~ \| f - h\|_\infty \leq \epsilon \big\},  \]
where $\| \cdot \|_\infty$ denotes the sup-norm defined in (\ref{def_norm_inf}).
\item Let $\mathcal{H}$ be a set of functions from $E_1$ to $E_2$. For any $\epsilon > 0$, the $\epsilon$-covering number $\mathcal{N}(\mathcal{H}, \epsilon)$  of $\mathcal{H} $, represents the minimal 
number of balls of radius $\epsilon$ needed to cover $\mathcal{H}$; that is,
\begin{equation}\label{epsi_covering_number}
 \mathcal{N}(\mathcal{H}, \epsilon) = \inf\Big\{ m \geq 1 ~: \exists h_1, \cdots, h_m \in \mathcal{H} ~ \text{such that} ~ \mathcal{H} \subset \bigcup_{i=1}^m B(h_i, \epsilon) \Big\}.
\end{equation}
\item For any $h: E_1 \rightarrow E_2$ and $U \subseteq E_1$, set
\begin{equation}\label{def_norm_inf}
\| h\|_{\infty} = \sup_{x \in E_1} \| h(x) \|_{E_2}, ~ \| h\|_{\infty,U} = \sup_{x \in U} \| h(x) \|_{E_2}.
\end{equation}  
\item  For two sequences of real numbers $(a_n)$ and $(b_n)$, we write $ a_n \lesssim b_n$ or $ b_n \gtrsim a_n$ if there exists a constant $C > 0$ such that $a_n \leq  C b_n$ for all $n \in \N$; $a_n \asymp b_n$ if $a_n \lesssim b_n$ and $a_n \gtrsim b_n$. 
\item For any $x \in \R^d$, $x^T$ denotes the transpose of $x$.
\item For any $x \in \R$, $\lfloor x \rfloor$ denotes the greatest integer less than or equal to $x$, and $\lceil x \rceil$ the least integer greater than or equal to $x$.  
\end{itemize}
In the sequel, we will deal with strongly mixing processes, see \cite{rosenblatt1956central}, \cite{vidyasagar2013learning}.

\begin{Def}
Let $\mathcal{Z} = \{Z_i\}_{i\in \Z}$ be a stationary process on a probability space $(\Omega, \mathcal{B}, P)$. For $-\infty \leq i \leq \infty$, let $\sigma_{i}^{\infty}$ and $\sigma_{-\infty}^i$ be the sigma-algebras generated respectively by the random variables $Z_j, j \ge i$ and $Z_j, j\leq i$. The process $\mathcal{Z}$ is said to be $\alpha$-mixing, or strongly mixing, if
\begin{equation*}
\underset{A \in \sigma_{-\infty}^0, B \in\sigma_k^{\infty}}{sup} \{ |P(A \cap B) - P(A)P(B) |\} = \alpha(k) \to 0 \text{ when } k \to \infty. 
\end{equation*}
$\alpha(k)$ is then called the strong mixing coefficient.

\end{Def}
Let us recall the definition of piecewise linear and locally quadratic functions, see also  \cite{ohn2019smooth}, \cite{ohn2022nonconvex}.
\begin{Def}\label{def_pwl_quad}
Let a function $g: \R \rightarrow \R$.
\begin{enumerate}
\item $g$ is continuous piecewise linear (or \quote{piecewise linear} for notational simplicity) if it is continuous and there exists $K$ ($K\in \N$) break points $a_1,\cdots, a_K \in \R$ with $a_1 \leq a_2\leq\cdots \leq a_K $ such that, for any $k=1,\cdots,K$, $g'(a_k-) \neq g'(a_k+)$ and $g$ is linear on $(-\infty,a_1], [a_1,a_2],\cdots [a_K,\infty)$.
\item $g$ is locally quadratic if there exits an interval $(a,b)$ on which $g$ is three times continuously differentiable with bounded derivatives and there exists $t \in (a,b)$ such that $g'(t) \neq 0$ and $g''(t) \neq 0$.
\end{enumerate} 
\end{Def}

\medskip

Consider the process $\{ Z_t=(X_t, Y_t), t \in \Z \}$ which take values in $ \mathcal{Z}=   \mathcal{X} \times \mathcal{Y} \subset \R^d \times \R$ and generated from the model (\ref{equa_model_reg}), the loss function 
$\ell : \R \times \mathcal{Y} \to [0, \infty)$ and an activation function $\sigma: \R \rightarrow \R $, and set the following assumptions. 
\begin{itemize}
\item[\textbf{(A0)}:] $\mathcal{X}\subset R^d$ is a compact set.
\item[\textbf{(A1)}:] There exists a constant $C_{\sigma}>0$ such that the activation function $\sigma$ is $C_{\sigma}$-Lipschitz, that is, there exists $C_{\sigma} > 0$ such that $|\sigma(x_1) - \sigma(x_2)| \leq C_{\sigma} |x_1 - x_2|$ for any $x_1, x_2 \in \R$.
Moreover, $\sigma$ is  either piecewise linear or locally quadratic and fixes a non empty interior segment $I \subseteq [0,1]$ (i.e. $\sigma(z) = z$ for all $z \in I$).
\item[\textbf{(A2)}:] The process $\{Z_t = (X_t, Y_t ),  t \in \Z \}$ is stationary and ergodic, and strong mixing  with the mixing coefficients defined by
\begin{equation}\label{coef_alpha_mixing}
\alpha(j) \leq \overline{\alpha} \exp(-c j), ~ j \ge 1, \overline{\alpha} > 0, c > 0.
\end{equation}
\item[\textbf{(A3)}:] 
Local structure of the excess risk: There exist three constants $\mk_0:= \mk_0(Z_0, \ell, h^*) , \epsilon_0:= \epsilon_0(Z_0, \ell, h^*) >0$ and $\kappa:=\kappa(Z_0, \ell, h^*) \geq 1$ such that,  
\begin{equation}\label{assump_local_quadr}
 R(h) - R(h^*) \leq \mk_0 \| h - h^*\|^\kappa_{\kappa, P_{X_0}}, 
\end{equation}
for any measurable function $h: \R^d \rightarrow \R$ satisfying $\|h - h^*\|_{\infty, \mx} \leq \epsilon_0$; where $P_{X_0}$ denotes the distribution of $X_0$ and 
\[ \| h - h^{*}\|_{r, P_{X_0}}^r := \displaystyle \int | h (\text{x})- h^{*} (\text{x}) |^r  d P_{X_0} ( \text{x}),  
\]
for all $r \geq 1$.

\item[\textbf{(A4)}:] There exists $\mathcal{K}_f>0$ such that, the density function $f$ is $\mathcal{K}_f$-Lipschitz and has almost everywhere a continuous first-order derivative $f^{(1)}$.
Moreover, for any constant $C>0$, there exists $\mathcal{K} = \mathcal{K}(C) >0$ such that, for any real-valued random variable $U$ we have:
\[ \E|U| \leq C  \Rightarrow  \E\big[ |f^{(1)}(U)/f(U)| \big] \leq \mathcal{K} .\]
Recall that, $f$ is the density of the error variable $\xi_0$ in the model (\ref{equa_model_reg}).
\item[\textbf{(A5)}:] For any $C > 0$, there exists $p = p(C) \ge 1$ such that, for any real-valued random variable $U$ we have: 
\[ |U| \leq C ~a.s.  \Rightarrow \E\Big[\Big|\log f(\epsilon_0 + U) \big) \ind_{\{f(\epsilon_0 + U) \leq \beta_n \}} \Big| \Big] \lesssim \dfrac{\log n}{n},
\]
where $\beta_n = n^{-p}$ for $n$ sufficiently large.
\end{itemize}

\noindent
Several classical activation functions, including the  ReLU (rectified linear unit) given by $\sigma(x) = \max(x,0)$, satisfy the assumption \textbf{(A1)}, see also \cite{kengne2025excess}.
The strongly mixing condition \textbf{(A2)} is satisfied by many commonly used autoregressive processes, see for instance \cite{chen2000geometric}, \cite{doukhan1994mixing}.
\textbf{(A3)} is a local structure condition on the excess risk. This condition holds for example, when the error $\xi_0$ belongs to the class of Subbotin distribution, see Proposition \ref{equa_bound_assumpA5} below.
Moreover, one can easily see that \textbf{(A4)} is satisfied for this class of distribution.  
See Proposition \ref{equa_bound_assumpA5} which shows that \textbf{(A5)} is also satisfied.

\medskip

A real-valued random variable $U$ follows a Subbotin distribution with parameter $r >0$ if its density (with respect to the Lebesgue measure) is given by:
\begin{equation}\label{def_Subbotin}
\text{for any } u\in \R, f_U(u) = C_r \exp \big(-|u|^r/r \big), \text{ for some constant positive } C_r > 0.
\end{equation}
This class of distribution was introduced by \cite{subbotin1923law} and includes the Laplace distribution with $r= 1$ and the Gaussian (standard normal) density with $r= 2$.
The following Proposition \ref{equa_bound_assumpA5} shows that if the density $f(\cdot)$ of the error term $\xi_0$ in the model (\ref{equa_model_reg}) is Subbotin, then \textbf{(A3)} and \textbf{(A5)} are satisfied. 
\begin{prop}\label{equa_bound_assumpA5}
    Assume that the density $f$ of $\xi_0$ is Subbotin with parameter $r >0$. Then:
  \begin{enumerate}
  \item If $r \in (0,2]$, then \textbf{(A3)} holds with $\kappa = r$.
  \item \textbf{(A5)} is satisfied for any $p > 2^r$.
  \end{enumerate}
\end{prop}

\subsection{ Deep neural Networks }\label{Def_DNNs}
Let $L \in \N$ be the number of hidden layers or depth and $\mathbf{p} = (p_0, p_1,\cdots, p_{L+1}) \in \N^{L+2}$ a vector of widths. 
A DNN function with network architecture $ (L, \textbf{p}) $ is a function of the form:
\begin{equation}\label{h_equ1}
h: \R^{p_0} \rightarrow \R^{p_{L+1}}, \;   x\mapsto h(x) = A_{L+1} \circ \sigma_{L} \circ A_{L} \circ \sigma_{L-1} \circ \cdots \circ \sigma_1 \circ A_1 (x),
\end{equation} 
where $ A_j: \R^ {p_ {j -1}} \rightarrow \R^ {p_j} $ with $A_j (x) \coloneqq W_j x + \textbf{b}_j$ is a linear affine map, for given  $p_ {j}\times p_{j-1}$  weight matrix $ W_j$   and a shift vector $ \textbf{b}_j \in \R^ {p_j} $, and $\sigma_j: \R^{p_j} \rightarrow \R^ {p_j} $ is a nonlinear element-wise activation map, defined for all $z=(z_1,\cdots,z_{p_j})^T$ by $\sigma_j (z) = (\sigma(z_1), \cdots, \sigma(z_{p_j}))^{T} $.
Recall that, $^T$ denotes the transpose.   
Denote by,
\begin{equation} \label{def_theta_h}
\theta(h) \coloneqq \left(vec(W_1)^ {T}, \textbf{b}^{T}_{1}, \cdots,  vec(W_{L + 1})^{T} , \textbf{b}^ {T}_{L+1} \right)^{T}, 
 \end{equation} 
the network weights (or the vector of parameters) of a DNN function of the form (\ref{h_equ1}), where $ vec(W)$ is the vector obtained by concatenating the column vectors of the matrix $W$. 
 We will deal with an activation function $ \sigma: \R \rightarrow \R$ and the class of DNN predictors with $p_0$-dimensional input and $p_{L+1} $-dimensional output, denoted by $\mathcal{H}_{\sigma, p_0, p_{L+1}} $.
For the framework considered here, $p_0 = d$ and $p_ {L + 1} = 1$.
For a DNN $h$ as in (\ref{h_equ1}), denote by depth($h$)$=L$ and  width($h$) = $\underset{1\leq j \leq L} {\max} p_j $ its depth and width respectively. For any positive constants $L, N, B, F > 0$, we set
\[ \mathcal{H}_{\sigma}(L, N, B) \coloneqq \big \{h\in \mathcal{H}_{\sigma, d, 1}: \text{depth}(h)\leq L, \text{width}(h)\leq N, \|\theta(h)\|_{\infty} \leq B \big\},  \]
 and
\begin{equation}\label{DNNs_no_Constraint}
\mathcal{H}_{\sigma}(L, N, B, F) \coloneqq \big\{ h: h\in H_{\sigma}(L, N, B), \| h \|_{\infty, \mathcal{X}} \leq F \big\}. 
\end{equation}
We will also considered the class of sparsity constrained DNNs with sparsity level $S > 0$, defined by: 
\begin{equation}\label{DNNs_Constraint}
\mathcal{H}_{\sigma}(L, N, B, F, S) \coloneqq \left\{h\in \mathcal{H}_{\sigma}(L, N, B, F) \; :  \; \| \theta(h) \|_0 \leq S 
  \right\},
\end{equation}
where $\| x \|_0 = \sum_{i=1}^p \ind(x_i \neq 0), ~ \| x\|_{\infty} = \underset{1 \leq i  \leq p}{\max} |x_i |$ for all $x=(x_1,\ldots,x_p)^{T} \in \R^p$ ($p \in \N$).
In the sequel, we will consider a setting where the network architecture parameters $L, N, B, F, S$ depend on the sample size $n$.

\section{Excess risk bounds for the MEE-based NPDNN}\label{excess_NPDNN}
In this section and the next one, we will establish bounds of the excess risk, given for all predictor $h$ by $\mathcal{E}_{Z_0}(h) = R(h) - R(h^{*})$, see (\ref{equa_excess_risk}).
In many common settings (see for instance Proposition \ref{prop_minimizer_risk} below), the regression function $h_0$ in (\ref{equa_model_reg}) is a target function. In these cases $\mathcal{E}_{Z_0}(h) = R(h) - R(h_0)$. 

\medskip

\begin{prop}\label{prop_minimizer_risk}
    Assume that the density function $f$ of the error $\xi_0$ is symmetric and supported on $\R$. Then, the unknown regression function $h_0$ in (\ref{equa_model_reg}) is a target function. That is, $h_0$ minimizes the Shannon's entropy risk defined in (\ref{equa_def_risk}).
\end{prop}

\noindent
Let us consider the class of H\"older smooth functions.
\medskip

\noindent
\textbf{Class of H\"older functions}: Let $ D\subset \R^d, \beta, A >0$. The set of $\beta$-H\"older functions with radius $A$ defined on $D$ is given by:
\begin{equation}\label{equa_ball_Holder}
 \mathcal{C}^{\beta}(D, A)= \Bigg\{h : D \rightarrow \R : \underset{\boldsymbol{\alpha} : |\boldsymbol{\alpha}|_1 < \beta}{\sum}\|\partial^{\boldsymbol{\alpha}}h\|_{\infty} + \underset{\boldsymbol{\alpha} : |\boldsymbol{\alpha}|_1 = \lfloor \beta \rfloor}{\sum}~\underset{x\ne y}{\underset{x, y\in D}{\sup}} \frac{|\partial^{\boldsymbol{\alpha}}h(x) - \partial^{\boldsymbol{\alpha}}h(y)|}{|x - y|^{\beta - \lfloor \beta \rfloor}} \leq A \Bigg\},  
\end{equation}
with $\boldsymbol{\alpha} = (\alpha_1, \dots, \alpha_d)\in \N^d$, $|\boldsymbol{\alpha}|_1 \coloneqq \sum_{i=1}^d \alpha_i$ and $\partial^{\boldsymbol{\alpha}} = \partial^{\alpha_1} \dots \partial^{\alpha_d}$.

\medskip

\noindent
The following theorem establishes a bound of the expected excess risk of the MEE-based NPDNN estimator on the class of H\"older functions. 
\begin{thm} \label{excess_risk_bound__expo_strong_mixing} 
Assume that (\textbf{A0}), (\textbf{A1}), (\textbf{A2}), \textbf{(A4)}, \textbf{(A5)} holds. Assume that \textbf{(A3)} is satisfied for some $\kappa \geq 1$, $\mk_0, \varepsilon_0>0$ and that there exist $C_0 > 0$ such that $\|h_0\|_{\infty} < C_0$ where $h_0$ is given in $(\ref{equa_model_reg})$.
Let $s, \mathcal{K}^{*} >0$.
Consider the class of DNN $\mathcal{H}_{\sigma}(L_n, N_n, B_n, F_n, S_n)$ defined in (\ref{DNNs_Constraint}), with
$L_n = \dfrac{s L_0}{\kappa s + d} \log n, N_n = N_0 n^{\frac{d}{ \kappa s + d}}, S_n = \frac{s S_0}{\kappa s + d} n^{\frac{d}{\kappa s + d}}\log n, 
$
 $ B_n = B_0 n^{\frac{4(d + s)}{\kappa s + d}} $ and $F_n = F>0$, 
for some positive constants $L_0, N_0, S_0, B_0 > 0$. 
Then, there exists $n_0 = n_0(\kappa, L_0, N_0, B_0, S_0, C_\sigma,$ $ s, d) $ defined at (\ref{equa_limite}), such that for all $n \geq n_0$, the MEE-based NPDNN estimator $\widehat{h}_{n,NP}$ defined in  (\ref{Non_Pen_DNNs_Estimators}) satisfies for all $\nu>6$,
\begin{equation}\label{thm2_excess_risk_bound}
 \sup_{h^{*} \in \mathcal{C}^{s}(\mathcal{X}, \mathcal{K}^{*}) } \Big(  \E[R(\widehat{h}_{n,NP}) - R(h^{*})] \Big) 
\lesssim \dfrac{(\log n)^{\nu}}{ n^{\frac{\kappa s}{\kappa s + d} }},
\end{equation}
where the $\sup$ is taken for target function given in (\ref{best_pred_F}).   
\end{thm}
\noindent
It follows from Theorem \ref{excess_risk_bound__expo_strong_mixing} and Proposition \ref{equa_bound_assumpA5} and \ref{prop_minimizer_risk} that if the error $\xi_0$ belongs to the class of Subbotin distribution with parameter $r \in (0,2]$, then the convergence rate of $\E[R(\widehat{h}_{n,NP}) - R(h_0)]$ is of order $\mathcal{O}\big( n^{-\frac{r s}{r s + d} }  \log^6 n\big)$. 
In the sequel, we also deal with the class of composition H\"older functions.

\medskip

\noindent
\textbf{Class of composition H\"older functions.}\label{excess_risk_comp_Holder}
Let $q\in \N, \boldsymbol{d} = (d_0, \dots, d_{q+1})\in \N^{q + 2}$ with $d_0 = d, d_{q+1} = 1, \boldsymbol{t} = (t_0, \dots, t_q)\in \N^{q+1}$ with $t_i \leq d_i$ for all $i$, $\boldsymbol{\beta} = (\beta_0, \dots, \beta_q)\in (0,\infty)^{q+1}$ and $A>0$. 
For all $l < u$, denote by $\mathcal{C}_{t_i}^{\beta_i}([l, u]^{d_i}, A)$ the set of functions $f: [l, u]^{d_i} \rightarrow \R$ that depend on at most $t_i$ coordinates and $f \in \mathcal{C}^{\beta_i}([l, u]^{d_i}, A)$.

\medskip

\noindent 
Let us define the set of composition H\"older functions functions $\mathcal{G}(q, \bold{d}, \bold{t}, \boldsymbol{\beta}, A)$, see \cite{schmidt2020nonparametric}, by: 
\begin{multline}\label{equa_compo_structured_function}
\mathcal{G}(q, \bold{d}, \bold{t}, \boldsymbol{\beta}, A) \coloneqq \Big\{h = g_q\circ \dots \circ g_0, g_i = (g_{ij})_{j=1,\cdots,d_{i+1}} : [l_i, u_i]^{d_i}
 \rightarrow [l_{i+1}, u_{i+1}]^{d_{i+1}}, g_{ij}\in C_{t_i}^{\beta_i}
 ([l_i, u_i]^{d_i}, A) ~ \\
        \text{for some } l_i, u_i \in \R \text{ such that } ~ |l_i|, |u_i| \leq A, \text{ for } i=1,\cdots,q \Big\}.
\end{multline}
$\beta_i^* \coloneqq \beta_i\prod_{j = i+1}^q(\beta_{j} \land 1)$ represents the smoothness of a composition function in $\mathcal{G}(q, \bold{d}, \bold{t}, \boldsymbol{\beta}, A)$ (see also \cite{juditsky2009nonparametric}).
In the sequel, we set:
\begin{equation}\label{def_phi_nalpha}
\phi_{n} \coloneqq  \underset{0\leq i \leq q}{\max} n^{ \frac{- 2\beta_i^* }{2\beta_i^* + t_i}   }.
\end{equation}

\medskip

\medskip

\noindent
The following theorem is established with the ReLU activation function  and $\mathcal{X} = [0,1]^d$.
\begin{thm}\label{excess_risk_bound__expo_strong_mixing_thm1_v2}
Set $\mathcal{X} = [0,1]^d$. Assume that \textbf{(A2)}, \textbf{(A4)}, \textbf{(A5)} and  \textbf{(A3)} for some $\kappa \ge 1, \mk_0 , \varepsilon_0 > 0$ hold.
Consider the class of H\"older composition functions  $ \mathcal{G}(q, \bold{d}, \bold{t}, \boldsymbol{\beta}, \mk^{*})$ defined in (\ref{equa_compo_structured_function}) and the DNN class $\mathcal{H}_{\sigma}(L_n, N_n, B_n, F_n, S_n)$ with the ReLU activation function $\sigma$ and where the network architecture parameters $(L_n, N_n, B_n, F_n, S_n)$ satisfy $L_n\asymp \log n, N_n \asymp n \phi_n, B_n = B \geq 1 , F_n = F > \max(\mathcal{K}^{*}, 1), S_n \asymp n \phi_n \log n$.
Then, the MEE-based NPDNN estimator $\widehat{h}_{n,NP}$ defined in (\ref{Non_Pen_DNNs_Estimators}) satisfies for all $\nu > 6$, 
\begin{align}
\sup_{h^{*} \in \mathcal{G}(q, \bold{d}, \bold{t}, \boldsymbol{\beta}, \mk^{*}) } \Big(  \E[R(\widehat{h}_{n,NP}) - R(h^{*})] \Big)  \lesssim \Big(\phi_{n}^{\kappa/2} \lor \phi_n \Big) (\log n)^{\nu}, 
\end{align}
where $\phi_n$ is given in (\ref{def_phi_nalpha}) and the $\sup$ is taken for target function given in (\ref{best_pred_F}). 
\end{thm}
 \noindent
 Therefore, if the error $\xi_0$ in the model (\ref{equa_model_reg}) is Gaussian, then we get from Proposition \ref{prop_minimizer_risk} and Theorem \ref{excess_risk_bound__expo_strong_mixing_thm1_v2} that   $\E[ \| \widehat{h}_{n,NP} - h_0\|_{2, P_{X_0}}^2 ] \lesssim  \phi_n (\log n)^{6}$ (see also Section \ref{example} below).
 This rate coincides (up to a logarithmic factor) with the well-known rate obtained from i.i.d. data in this case, and is minimax optimal (up to a logarithmic factor), see \cite{schmidt2020nonparametric}.

\section{Excess risk bounds for the MEE-based SPDNN}\label{excess_risk_SPDNN}

The following Theorem provides an upper bound of the expected excess risk of the MEE-based SPDNN estimator.

\begin{thm}\label{excess_risk_bound__expo_strong_mixing_thm2}
Assume that \textbf{(A1)}, \textbf{(A2)} and \textbf{(A5)} hold and that there exists $C_0>0$ such that, $\|h_0\|_{\infty} \leq C_0$, where $h_0$ is given in $(\ref{equa_model_reg})$. 
Consider the class of DNN predictors $\mathcal{H}_{\sigma}(L_n, N_n, B_n, F)$ with $ L_n \asymp\log(n), N_n \lesssim n^{\nu_1}, 1 \leq B_n \lesssim n^{\nu_2}, F_n = F > 0$, for some $ \nu_1 > 0, \nu_2 > 0$.  
Let $\nu_3 > 5$.
Then, there exists $n_0 \in \N$ such that, for all $n \geq n_0$, the MEE-based SPDNN estimator $\widehat{h}_{n,SP}$ defined in (\ref{Pen_DNNs_Estimators_v1}),
with $\lambda_n \asymp \frac{\big(\log n
 \big) ^{\nu_3}}{n}, \tau_n \leq \dfrac{1}{16 \mathcal{K}_{\ell}(L_n + 1)(n(N_n + 1)B_n)^{L_n + 1} }$ satisfies
\begin{equation}
 \E[R(\widehat{h}_{n,SP}) - R(h^*) ]  \lesssim   \dfrac{\big(\log n\big)^3}{n} + 2\big(R(\widetilde{h}_{\mathcal{H}_{\sigma, n}}) - R(h^*) + J_n(\widetilde{h}_{\mathcal{H}_{\sigma, n}})\big),
\end{equation}
for any target function $h^*$ in (\ref{best_pred_F}) with $\|h^{*}\|_{\infty} \leq \mathcal{K}^*$ for some $\mathcal{K}^* >0$,
and for all 
\begin{equation}\label{def_target_tilde_spdnn}
 \widetilde{h}_{\mathcal{H}_{\sigma, n}} \in  \underset{h\in \mathcal{H}_{\sigma}(L_n, N_n, B_n, F)}{\argmin} \Big[R(h) + J_n(h) \Big].  
\end{equation}
\end{thm}
\noindent
From this theorem, the following corollary which establishes an oracle inequality for the predictor $\widehat{h}_{n,SP}$, holds.
\begin{Corol}\label{equa_oracle_ineq_expo_strong_mixing}
Assume that \textbf{(A1)}, \textbf{(A2)}, and \textbf{(A5)} hold. Consider the DNN class $\mathcal{H}_{\sigma}(L_n, N_n, B_n, F)$ where $ L_n, N_n, B_n, F_n$, and $\lambda_n, \tau_n$ satisfy the conditions in Theorem \ref{excess_risk_bound__expo_strong_mixing_thm2}. 
 Then, there exists $n_0 \in \N$ and a positive constant $C_1 > 0$ such that, for all $n \geq n_0$, the MEE-based estimator defined in (\ref{Pen_DNNs_Estimators_v1}), satisfies for any target function $h^{*}$ given in (\ref{best_pred_F}) such that $\|h^{*}\|_{\infty} \leq F$,
 \begin{align}\label{proof expected_excess_risk}
\ \E[R(\widehat{h}_{n,SP}) - R(h^*) ]
  \leq C_1\Bigg(\underset{h \in \mathcal{H}_{\sigma, n}} {\inf} \Big[ R(h) - R(h^*) + J_n(h) \Big]    +  \dfrac{\big(\log n\big)^3}{n} \Bigg).
\end{align}
\end{Corol}

\medskip

\noindent
The Corollary \ref{excess_risk_bound_class_of_Holder_function_V1} below provides an upper bound of the expected excess risk of the MEE-based SPDNN estimator on the class of H\"older smooth functions.
 
\medskip

\begin{Corol}\label{excess_risk_bound_class_of_Holder_function_V1}
Assume that \textbf{(A0)}, \textbf{(A1)}, \textbf{(A2)}, \textbf{(A5)} and $\textbf{(A3)}$ for some $\kappa \ge 1, \mk_0, \varepsilon_0 > 0$ hold.
Let $s, \mathcal{K}^{*} > 0$.
Consider the DNN class $\mathcal{H}_{\sigma}(L_n, N_n, B_n, F)$ where the network architecture $(L_n, N_n, B_n, F)$ are such that $L_n \asymp \log n, N_n \gtrsim n^{\frac{d}{\kappa s + d}}, B_n \gtrsim n^{\frac{4(s + d)}{\kappa s + d}}, F_n = F >0$ and the tuning parameters $\lambda_n, \tau_n$ that fulfill the conditions in Theorem (\ref{excess_risk_bound__expo_strong_mixing_thm2}).
Then, there exists $n_0 \in \N$ such that for all $n \geq n_0$, the MEE-based predictor $\widehat{h}_{n,SP}$ defined in (\ref{Pen_DNNs_Estimators_v1}) satisfies,
\begin{equation}\label{excess_risk_bound_v1}
 \sup_{h^{*} \in \mathcal{C}^{s}(\mathcal{X}, \mathcal{K}^{*}) } \Big(  \E[R(\widehat{h}_{n,SP}) - R(h^{*})] \Big) \lesssim \dfrac{  \big(\log n \big)^{\nu} }{n^{\kappa s/(\kappa s + d)}},
\end{equation}
for all $\nu > 5$ and where the $\sup$ is taken for target functions given in (\ref{best_pred_F}) that belong to the class of H\"older smooth functions defined in (\ref{equa_ball_Holder}).
 
\end{Corol}

\medskip

\noindent
The next corollary is established with the ReLU activation function and $\mathcal{X} = [0,1]^d$.
\begin{Corol}\label{excess_risk_bound_class_of_Holder_function_V2}
Set $\mathcal{X} = [0,1]^d$. Assume that \textbf{(A2)}, \textbf{(A5)} and  \textbf{(A3)} for some $\kappa \ge 1, \mk_0 , \varepsilon_0 > 0$ hold.
Consider a class of H\"older composition functions  $ \mathcal{G}(q, \bold{d}, \bold{t}, \boldsymbol{\beta}, \mk^{*})$ defined in (\ref{equa_compo_structured_function}) and the class of DNN predictor $\mathcal{H}_{\sigma}(L_n, N_n, B_n, F)$ with the ReLU activation function $\sigma$ with the network architecture $(L_n, N_n, B_n, F)$ satisfying $L_n\asymp \log n, N_n \gtrsim n \phi_n, B_n = B \geq 1 , F_n = F > \max(\mk^{*}, 1)$.
Then, there exists $n_0 \in \N$ such that for all $n \geq n_0$, the MEE-based predictor $\widehat{h}_{n,SP}$ defined in (\ref{Pen_DNNs_Estimators_v1}) with $\lambda_n, \tau_n$ chosen as in Theorem (\ref{excess_risk_bound__expo_strong_mixing_thm2}) satisfies for all $\nu > 6$, 
\begin{equation}\label{excess_risk_bound_corol01_v1}
\sup_{h^{*} \in \mathcal{G}(q, \bold{d}, \bold{t}, \boldsymbol{\beta}, \mk^{*}) } \Big(  \E[R(\widehat{h}_{n,SP}) - R(h^{*})] \Big)  \lesssim \Big( \phi_{n}^{\kappa/2} \lor \phi_n \Big)(\log n)^{\nu},  
\end{equation}
where $\phi_n$ is given in (\ref{def_phi_nalpha}) and the $\sup$ is taken for target functions given in (\ref{best_pred_F}) that belong to the class of composition H\"older functions defined in (\ref{equa_compo_structured_function}).
\end{Corol}

\noindent
As in the case of the MEE-based NPDNN estimator above, when the error $\xi_0$ is Gaussian, it holds from Proposition \ref{prop_minimizer_risk} and Corollary \ref{excess_risk_bound_class_of_Holder_function_V2} that   $\E[ \| \widehat{h}_{n,SP} - h_0\|_{2, P_{X_0}}^2 ] \lesssim  \phi_n (\log n)^{6}$.
 This rate coincides (up to a logarithmic factor) with that obtained from i.i.d. data and is optimal (up to a logarithmic factor) in the minimax sense.

\section{Examples}\label{example}
We consider a stationary and ergodic process $\{Z_t =(X_t, Y_t), t \in \Z \} $ with values in $ \mathcal{X} \times \mathcal{Y} \subset \R^d \times \R$ and solution of,
\begin{equation}\label{equa_model_reg2}
Y_t = h_0(X_t) + \xi_t, ~ X_t \sim P_{X_t}
\end{equation}
where $h_0 : \R ^d \to \R$ is the unknown regression function, $(\xi_t)_{t \in \Z}$ is an i.i.d. error process and $\xi_t$ is independent of $X_t$.
Assume that the density $f$ of $\xi_0$ belongs to the class of Subbotin distribution with parameter $r \in (0,2]$ (see (\ref{def_Subbotin})) and that the process $(Z_t)_{t \Z}$ satisfies the assumption \textbf{(A2)}.
For example, in the case of nonparametric autoregression (with $X_t=(Y_{t-1},\cdots,Y_{t-d})$) and with a Lipschitz-type condition on $h_0$, then the process $(Y_t)_{t \in \Z}$ is geometrically strongly mixing, see for instance \cite{doukhan1994mixing}, \cite{chen2000geometric}, \cite{chen2018generalized}, and \textbf{(A2)} holds. 
We deal with a framework that goes beyond the autoregression.
For this example, \textbf{(A3)} (with $\kappa=r$), \textbf{(A4)} and \textbf{(A5)} hold, see for instance Proposition \ref{equa_bound_assumpA5}.
Therefore, the bounds in Theorem \ref{excess_risk_bound__expo_strong_mixing}, \ref{excess_risk_bound__expo_strong_mixing_thm1_v2} and Corollary \ref{excess_risk_bound_class_of_Holder_function_V1}, \ref{excess_risk_bound_class_of_Holder_function_V2}
are applied with $\kappa = r$.

\medskip
\noindent
Moreover, $h_0$ is a target function (see Proposition \ref{prop_minimizer_risk}) and the excess risk is given for any predictor $h: \mx \rightarrow \R$ by,
\begin{equation}\label{excess_risk_Subbotin}
\mathcal{E}_{Z_0}(h) = R(h) - R(h_0) = \E [ | Y_0 - h(X_0) |^r -| Y_0 - h_0(X_0) |^r ] . 
\end{equation}
When $r=2$ (the case of Gaussian error), $\mathcal{E}_{Z_0}(h) = \E\big[ \big( h(X_0) - h_0(X_0) \big)^2 \big] = \| h - h_0\|^2_{2, P_{X_0}} $. 
In this case, if $h_0 \in \mathcal{C}^{s}(\mathcal{X}, \mathcal{K}^{*})$ for some $s, \mk^* >0$, then it holds from Theorem \ref{excess_risk_bound__expo_strong_mixing}, Corollary \ref{excess_risk_bound_class_of_Holder_function_V1} that the convergence rate of the $L_2$ errors $\E \big[\|\widehat{h}_{n,NP} - h_0\|^2_{2, P_{X_0}} \big]$ and $\E \big[\|\widehat{h}_{n,SP} - h_0\|^2_{2, P_{X_0}} \big]$ are $\mathcal{O}\big( n^{-\frac{2 s}{2 s + d} }  \log^6 n\big)$ and $\mathcal{O}\big( n^{-\frac{2 s}{2 s + d} }  \log^5 n\big)$ respectively.  
Furthermore, when $r=2$ and $h_0 \in \mathcal{G}(q, \bold{d}, \bold{t}, \boldsymbol{\beta}, \mk^{*})$ (defined in (\ref{equa_compo_structured_function})), it holds from Theorem \ref{excess_risk_bound__expo_strong_mixing_thm1_v2}, Corollary \ref{excess_risk_bound_class_of_Holder_function_V2} that the convergence rate of the $L_2$ errors $\E \big[\|\widehat{h}_{n,NP} - h_0\|^2_{2, P_{X_0}} \big]$ and $\E \big[\|\widehat{h}_{n,SP} - h_0\|^2_{2, P_{X_0}} \big]$ is $\mathcal{O}\big( \phi_n (\log n)^{6} \big)$, where $\phi_n$ is given in (\ref{def_phi_nalpha}).
As pointed out above, these rates are minimax optimal (up to a logarithmic factor).
For deep learning from dependent data, such optimal rates have been obtained for instance by \cite{ma2022theoretical} ($L_2$ loss and $\alpha$-mixing condition), \cite{kurisu2025adaptive} ($L_2$ loss and $\beta$-mixing processes), \cite{kengne2025deep} (Huber loss and $\alpha$-mixing condition).
The results above establish in particular the minimax optimal bound for the estimator $\widehat{h}_{n,SP}$, from strongly mixing observations and with the $L_2$ loss.
Also, let us stress that this estimator is adaptive, in the sense that the network architecture $(L_n, N_n, B_n, F_n)$ can be chosen adaptively (the smoothness $h_0$ does not need to be known) to reach the minimax optimal convergence rate.

\section{Discussion}\label{discussion}

 This paper considers the Shannon entropy for deep nonparametric regression from strongly mixing data.
 Two MEE-based estimators with a sparsity constrained and a sparse-penalized regularization are proposed.
For both predictors, the convergence rates of the expected excess risk over the classes of H\"older and composition H\"older functions are established.
Such estimators which are based on the log-likelihood loss function, take into consideration the information from the underlying distribution of data, and have the robustness to deal with non-Gaussian models or models with heavy-tailed noise.

\medskip

\noindent
However, the assumption that the density $f$ of the error $\xi_0$ is known is restrictive in practice. 
A natural approach is to deal with an estimator of this density.
For instance, one can perform a kernel estimation, and consider for any predictor $h : \mx \rightarrow \R$ with $\epsilon_i = Y_i - h(X_i)$ for $i=1,\cdots,n$, the estimator,
\begin{equation}\label{equa_kernel_estimator}
\widehat{f}_{n,h}(v) = \dfrac{1}{n} \sum_{i=1}^n \mathcal{K}_b(\epsilon_i, v),
\end{equation}
where $\mathcal{K}_b(y_1, y_2) = K((y_1 - y_2)/b)/b$, $b$ is a bandwidth parameter, and $K(\cdot)$ is a kernel function.
Therefore, we can consider the approximations of Shannon's entropy where the empirical version is given for all predictor $h$ by,
\[ \widetilde{R}_n(h) = - \dfrac{1}{n} \sum_{i=1}^{n} \Big(\log \widehat{f}_{n,h}(Y_i - h(X_i)) \Big) .\]
So, the corresponding MEE-based NPDNN and SPDNN estimators $\widetilde{h}_{n,NP}$, $\widetilde{h}_{n,SP}$ satisfy,
\begin{align}\label{sparse_DNNs_Estimators_v2_NP}
\nonumber \widetilde{h}_{n,NP} & \in \underset{h \in  \mathcal{H}_{\sigma}(L_n, N_n, B_n, F_n, S_n)}{\argmin} \left[ -\dfrac{1}{n} \sum_{i=1}^{n} \Big(\log \widehat{f}_{h}(Y_i - h(X_i)) \Big) \right] \\
& = \underset{h \in  \mathcal{H}_{\sigma}(L_n, N_n, B_n, F_n, S_n)}{\argmin} \left[ -\dfrac{1}{n} \sum_{i=1}^{n} \Bigg(\log \Big( \dfrac{1}{n}\sum_{j=1}^n \mathcal{K}_b \big(Y_j - h(X_j), Y_i - h(X_i) \big) \Big) \Bigg) \right],
\end{align}
and 
\begin{align}\label{sparse_DNNs_Estimators_v2_SP}
\nonumber \widetilde{h}_{n,SP} & \in \underset{h \in  \mathcal{H}_{\sigma}(L_n, N_n, B_n, F_n)}{\argmin} \left[ -\dfrac{1}{n} \sum_{i=1}^{n} \Big(\log \widehat{f}_{h}(Y_i - h(X_i)) \Big) + J_n(h) \right] \\
& = \underset{h \in  \mathcal{H}_{\sigma}(L_n, N_n, B_n, F_n)}{\argmin} \left[ -\dfrac{1}{n} \sum_{i=1}^{n} \Bigg(\log \Big( \dfrac{1}{n}\sum_{j=1}^n \mathcal{K}_b \big(Y_j - h(X_j), Y_i - h(X_i) \big) \Big) \Bigg) + J_n(h) \right],
\end{align}
where $L_n, N_n, B_n, F_n, S_n$ are the network architecture parameters, $\mathcal{H}_{\sigma}(L_n, N_n, B_n, F_n)$, $\mathcal{H}_{\sigma}(L_n, N_n, B_n, F_n, S_n)$ are the set of DNN predictors defined in (\ref{DNNs_no_Constraint}) and (\ref{DNNs_Constraint}), and $J_n(\cdot)$ denotes the penalty term given in (\ref{equa_penalty_term_v1}).   
 
 \medskip
 \noindent
Such approach has been considered by \cite{wang2025deep} from i.i.d. observations. 
As pointed out by \cite{li2025comment}, several aspects of the theoretical analysis in \cite{wang2025deep} remain open to discussion. 
Also, since the MEE estimators $\widehat{h}_{n, NP}$, $\widehat{h}_{n,SP}$ in (\ref{Non_Pen_DNNs_Estimators}) and (\ref{Pen_DNNs_Estimators_v1}), are based on the log-likelihood loss that takes into account the information on the distribution, one may wonder whether these predictors are optimal in terms of efficiency.
More precisely, when these estimators achieve the
minimax optimal rate (the case of Gaussian error for example), do they exhibit the lowest variance among 
all estimators that attain minimax optimal rate?
This issue has been addressed by \cite{wang2025deep} in the i.i.d. framework, although \cite{li2025comment} highlights some points in their proof that deserve further discussion. 
These two issues: (i) theoretical properties of the Shannon's MEE estimators based on the kernel estimation of the density, (ii) optimality in terms of efficiency, in particular from dependent data, are still a challenge.
So, we would like to share in this discussion, some ideas for the extension of this work.

\section{Proofs of the main results}\label{prove}
In the sequel, we will need the following lemma.

\begin{lem}\label{Lemma1}
  Consider model (\ref{equa_model_reg}). Assume that the density function $f$ of the error term $\xi_0$ admits a first-order derivative $f^{(1)}$ such that there exists a positive constant $\mathcal{K}_f$ satisfying 
  \[
    |f^{(1)}(u)| \leq \mathcal{K}_f, \quad \forall u \in \R.
  \]
 Let $n \in \N$, $0 < \beta_n \leq 1$ and denote by $T_{\beta_n} f(\cdot)$ the truncation of $f$ at level $\beta_n$ defined for all $u \in \R$ by,
 \begin{align}\label{equa_truncation_func_f}
T_{\beta_n}f(u) = \left\{
\begin{array}{ll}
  f(u)   & \text{ if } f(u) \ge \beta_n  \\
  \beta_n   &  \text{otherwise}.
\end{array}
\right.
\end{align}
  Then, for all $u_1, u_2 \in \R$, we have
  \begin{equation}\label{Proof_log_Tbeta_Lipschitz}
    \big|\log T_{\beta_n} f(u_1) - \log T_{\beta_n} f(u_2)\big|
    \leq \frac{\mathcal{K}_f}{\beta_n}\, |u_1 - u_2|,
  \end{equation}
\end{lem}

\begin{proof}
 First we remark that $\forall~ x_1, x_2 \ge \beta_n$, there exists $c$ between $x_1$ and $x_2$ such that
\begin{equation}\label{equa_Proof_lIpscitz_log}
    |\log x_1 - \log x_2| = \dfrac{1}{c}|x_1 - x_2| \leq \dfrac{1}{\beta_n}|x_1 - x_2| \leq |x_1 - x_2|.
\end{equation}
Let $u_1, u_2 \in \R$.
It holds from (\ref{equa_Proof_lIpscitz_log}) that,
\begin{itemize}
    \item[$\bullet$] If $f(u_1), f(u_2) \ge \beta_n$, we get
    \begin{align*}
         \big|\log T_{\beta_n} f(u_1) - \log T_{\beta_n} f(u_2)\big|  & = \big|\log f(u_1) - \log f(u_2)\big|   \leq \big|\log f(u_1) - \log f(u_2)\big|
    \\
    & \leq \frac{\mathcal{K}_f}{\beta_n}\, |u_1 - u_2|.
    \end{align*}
    \item[$\bullet$] If $f(u_1) \ge \beta_n$ and $f(u_2) \leq \beta_n$, we get
    \begin{align*}
         \big|\log T_{\beta_n} f(u_1) - \log T_{\beta_n} f(u_2)\big|  & = \big|\log f(u_1) - \beta_n\big|   \leq \dfrac{1}{\beta_n}\big|\log f(u_1) - \beta_n\big|
    \\
    & = \dfrac{1}{\beta_n}\big(\log f(u_1) - \beta_n\big) \leq \dfrac{1}{\beta_n}\big(\log f(u_1) - \log f(u_2)\big)
     \\
    &\leq \dfrac{1}{\beta_n}\big|\log f(u_1) - \log f(u_2)\big| \leq \dfrac{\mathcal{K}_f}{\beta_n}\big|u_1 - u_2\big|.
    \end{align*}
    In the case $f(u_1) \leq \beta_n$ and $f(u_2) \ge \beta_n$, and the case $f(u_1) \leq \beta_n$ and $f(u_2) \leq \beta_n$, the result above holds. This completes the proof of the Lemma. 
\end{itemize}

\end{proof}

\subsection{Proof of Proposition \ref{equa_bound_assumpA5}}
 \noindent
\textbf{1.} In this case, it holds from Proposition \ref{prop_minimizer_risk} that the regression function $h_0$ is a target function. So, we will establish (\ref{assump_local_quadr}) with $h^* = h_0$.
Let $\epsilon_0 > 0$. Consider $h: \R^d \rightarrow \R$ satisfying $\|h - h_0\|_{\infty, \mx} \leq \epsilon_0$.
We have,
\begin{equation}\label{proof_proof_Sb_delta_R}
  R(h) - R(h_0) =  \E [ -\log f(Y_0 - h(X_0)) +\log f(Y_0 - h_0(X_0)) ] = \E [ | Y_0 - h(X_0) |^r -| Y_0 - h_0(X_0) |^r ] . 
\end{equation} 
Let $ x \in \mx$. Since the distribution of $\xi_0$ is symmetric, one can easily see that, $\E[|Y_0 - h(X_0)|^r ~| X_0=x ] = \E[|2 h_0(X_0) - h(X_0) -Y_0|^r ~| X_0=x ]$.
Therefore,
\begin{align}
 \nonumber &\E [ | Y_0 - h(X_0) |^r -| Y_0 - h_0(X_0) |^r  ~| X_0 = x] = \E [ | Y_0 - h(x) |^r -| Y_0 - h_0(x) |^r  ~| X_0 = x] \\ 
 \nonumber  &= \E \Big[  \dfrac{|Y_0 - h(x) |^r + |2h_0(x) - h(x) - Y_0|^r}{2} - | Y_0 - h_0(x) |^r  ~| X_0 = x \Big] \\
 \nonumber  &= \E \Big[  \dfrac{| h_0(x) - h(x) + Y_0 - h_0(x) |^r + |h_0(x) - h(x) - (Y_0 - h_0(x) )|^r}{2} - | Y_0 - h_0(x) |^r  ~| X_0 = x \Big]\\ 
 \label{proof_proof_Sb_delta_R_cond} &\leq   | h(x) - h_0(x)|^r ,
\end{align} 
 where the (\ref{proof_proof_Sb_delta_R_cond}) holds from the Clarkson's inequality (see \cite{clarkson1936uniformly}) since $r \leq 2$.
Hence, $R(h) - R(h_0) \leq \E[ | h(X_0) - h_0(X_0)|^r ] = \| h - h_0\|^r_{r, P_{X_0}}$.

\medskip
\noindent 
\textbf{2.} Note that,
\begin{align}\label{equa_bound_indi_densi_v0}
    \E\Big[\Big| \Big(\log \beta_n - \log f(Y_0 - h(X_0)) \Big) \ind_{\big\{f(Y_0 - h(X_0)) \leq \beta_n\big\}}\Big|\Big] \leq 2 \log |\beta_n| \times  \E\Big[\ind_{\big\{f(Y_0 - h(X_0)) \leq \beta_n\big\}} \Big]. 
\end{align}
Let us compute $\E\big[\ind_{\big\{f(Y_0 - h(X_0)) \leq \beta_n\big\}} \big]$.
\begin{align}
\nonumber &\E\Big[\ind_{\big\{f(Y_0 - h(X_0)) \leq \beta_n\big\}} \big| X_0 = x\Big]  = \E\Big[\ind_{\big\{f(\epsilon_0 + h_0(X_0) - h(X_0)) \leq \beta_n\big\}} \big| X_0 = x\Big]     
\\
\nonumber & = P\Big\{ f\big(\epsilon_0 + h_0(X_0) - h(X_0) \big) \leq \beta_n  \big| X_0 = x \Big\}
 = P\Big\{ f\big(\epsilon_0 + h_0(x) - h(x) \big) \leq \beta_n \Big\}.
\end{align}
Recall the assumption that $f(\cdot)$ denotes the Subbotin density; that is, for any $y$, we have $f(y) = C_r \exp\big(-|y|^r/r \big)$ for some positive constant $C_r > 0$ and $r >0$.
Thus, we have
\begin{align}\label{equa_bound_indi_densi_v1}
\nonumber P\Big\{ f\big(\epsilon_0 + h_0(x) - h(x) \big) \leq \beta_n \Big\}  & = P\Big\{ C_r \exp\Big[-\frac{1}{r} \Big(\big|\epsilon_0 + h_0(x) - h(x) \big|^r \Big) \Big] \leq \beta_n \Big\}  
\\
 & =  P\Big\{ \big|\epsilon_0 + h_0(x) - h(x) \big|^r \ge -r\log \Big(\dfrac{\beta_n}{C_r} \Big) \Big\}.  
\end{align}
Note that, $ \big|\epsilon_0 + h_0(x) - h(x) \big| \leq  \big|\epsilon_0 \big| + \big|h_0(x) - h(x) \big|$, that's implies 
\[\Big\{\omega: \big|\epsilon_0(\omega) + h_0(x) - h(x) \big|^r \ge  -r\log \Big(\dfrac{\beta_n}{C_r} \Big) \Big\} \subseteq \Big\{\omega: \Big(\big|\epsilon_0(\omega)\big| + \big|h_0(x) - h(x) \big| \Big)^r \ge  -r\log \Big(\dfrac{\beta_n}{C_r} \Big) \Big\}.
\]
Hence, in addition to (\ref{equa_bound_indi_densi_v1}), we get
\begin{align}
\nonumber  P\Big\{ \big|\epsilon_0 + h_0(x) - h(x) \big|^r \ge -r\log \Big(\dfrac{\beta_n}{C_r} \Big) \Big\} & \leq    P\Big\{\Big(\big|\epsilon_0\big| + \big|h_0(x) - h(x) \big| \Big)^r \ge  -r\log \Big(\dfrac{\beta_n}{C_r} \Big) \Big\}
\\
& \leq  P\Big\{\big|\epsilon_0\big| \ge  \Big[-r\log \Big(\dfrac{\beta_n}{C_r} \Big) \Big]^{1/r} - \big|h_0(x) - h(x) \big| \Big\}.
\end{align}
One can see that, the tails of the Subbotin density $f(\cdot)$ are given as follows: for any $u \ge 1$, $P(|\epsilon_0| > u) \leq 2 C_r\exp \big(-\dfrac{u^r}{r} \big)$. 
Thus, we have
\begin{align}\label{equa_bound_indi_densi_v2}
\nonumber P\Big\{\big|\epsilon_0\big| \ge  \Big[-r\log \Big(\dfrac{\beta_n}{C_r} \Big) \Big]^{1/r} - \big|h_0(x) - h(x) \big| \Big\} & \leq 2C_r\exp \Big[-\dfrac{1}{r}\Big(\Big[-r\log \Big(\beta_n/C_r \Big) \Big]^{1/r} - \big|h_0(x) - h(x) \big| \Big)^{r} \Big]
\\
 & \leq 2\exp \Big[-\dfrac{1}{r}\Big\{\Big[-r\log \Big(\beta_n/C_r \Big) \Big]^{1/r} \Big(1 - \dfrac{\big|h_0(x) - h(x) \big|}{\Big[-r\log \Big(\beta_n/C_r\Big) \Big]^{1/r}} \Big) \Big\}^{r} \Big]. 
\end{align}
Recall that, $\big|h_0(x) - h(x) \big| \leq C_0 + F$, and $\Big[-r\log \Big(\beta_n/C_r \Big) \Big]^{1/r} \to \infty$ as $n \to \infty$. Thus, we have
\[\dfrac{\big|h_0(x) - h(x) \big|}{\Big[-r\log \Big(\beta_n/C_r\Big) \Big]^{1/r}} \leq \dfrac{C_0 + F}{\Big[-r\log \Big(\beta_n/C_r\Big) \Big]^{1/r}} \limiten 0.\]
Hence, there exists $n_0 = n_0(r, C_r)$ such that for any $n > n_0$, we have
\begin{equation}\label{equa_bound_indi_densi_v3}
\dfrac{C_0 + F}{\Big[-r\log \Big(\beta_n/C_r\Big) \Big]^{1/r}} \leq \dfrac{1}{2}.
\end{equation}
According to (\ref{equa_bound_indi_densi_v2}) and (\ref{equa_bound_indi_densi_v3}), we get
\begin{align}\label{equa_bound_indi_densi_v4}
\nonumber P\Big\{\big|\epsilon_0\big| \ge  \Big[-r\log \Big(\dfrac{\beta_n}{C_r} \Big) \Big]^{1/r} - \big|h_0(x) - h(x) \big| \Big\} & \leq 2C_r \exp \Big[-\dfrac{1}{r 2^r} \Big(\Big[-r\log \Big(\beta_n/C_r \Big) \Big]^{1/r} \Big)^r \Big]
\\
 & \leq 2C_r \exp \Big[\log \Big(\beta_n/C_r \Big)^{\dfrac{1}{2^r}} \Big]  \lesssim \Big(\beta_n\Big)^{\dfrac{1}{2^r}}.
\end{align}
Set $\beta_n = \dfrac{1}{n^p}$, with $p \ge 2^r$. In addition to (\ref{equa_bound_indi_densi_v4}), we have
\begin{align}\label{equa_bound_indi_densi_v5}
 P\Big\{\big|\epsilon_0\big| \ge  \Big[-r\log \Big(\dfrac{\beta_n}{C_r} \Big) \Big]^{1/r} - \big|h_0(x) - h(x) \big| \Big\} & \lesssim \dfrac{1}{n^{p/2^r}}.
\end{align}
Hence, according to (\ref{equa_bound_indi_densi_v0}) and (\ref{equa_bound_indi_densi_v5}), we have
\begin{align}
\E\Big[\Big| \Big(\log \beta_n - \log f(Y_0 - h(X_0)) \Big) \ind_{\big\{f(Y_0 - h(X_0)) \leq \beta_n \big\}}\Big|\Big] \lesssim \dfrac{\log n}{n}. 
\end{align}
Which completes the proof.
\qed

\medskip

\subsection{Proof of Proposition \ref{prop_minimizer_risk}}
Recall the assumption that the density function $f$ of the error term $\xi_0 = Y_0 - h_0(X_0)$ is symmetric. That is, $f$ satisfies
\begin{equation}\label{equa_proof_kernel_densitie_normalized}
    \displaystyle \int_\R f(u) du = 1 \text{ and }  \displaystyle \int_\R u f(u) du = 0.
\end{equation}
Moreover, $f$ is an even function and the density of $Y_0\big|X_0=x$ is $f_{Y_0\big|X_0=x}(y) = f(y - h_0(x))$ for all $x\in \mathcal{X}$. 
 Let us consider the function
\begin{equation}
    V(s) = \displaystyle \int_\R -\log \big(f(u - s) \big) f(u) du.
\end{equation}
Thus, for any $s > 0$, we have
\begin{align}\label{equa_proof_minimizer_V1}
 \nonumber V(s) - V(0) & =  \displaystyle \int_\R -\log \big(f(u - s) \big) f(u) du + \displaystyle \int_\R \log \big(f(u) \big) f(u) du   =  \displaystyle \int_\R f(u) \log \Big( \dfrac{f(u)}{f(u - s)} \Big) du \\
 &  = \E\Big[\log \Big(\dfrac{f(\xi_0)}{f(\xi_0 - s)} \Big)\Big]  =  \E\Big[-\log \Big(\dfrac{f(\xi_0 - s)}{f(\xi_0)} \Big)\Big] \ge - \log \E\Big[\dfrac{f(\xi_0 - s)}{f(\xi_0)} \Big] = 0, 
\end{align}
where (\ref{equa_proof_minimizer_V1}) is obtained by applying Jensen’s inequality.
Hence, according to (\ref{equa_proof_minimizer_V1}), the function $V$ achieves its minimum at $s = 0$.

\medskip

\noindent
Let $h: \mathcal{X} \to \R$, a measurable function.
We have,
\begin{align}\label{equa_bound_gain_risk_v1}
 \nonumber R(h) &= \E \Big[-\log f(Y_0 - h(X_0)) \Big] = \displaystyle \int_{\mathcal{X}} \displaystyle \int_{-\infty}^{\infty} -\log f(y - h(x)) f_{Y_0\big|X_0=x}(y)  dy d P_{X_0}(x)   \\
 & = \displaystyle \int_{\mathcal{X}} \displaystyle \int_{-\infty}^{\infty}-\log f(y - h(x)) f(y - h_0(x)) dy d P_{X_0}(x).
\end{align}
Set $u = y - h_0(x)$.
Thus, (\ref{equa_bound_gain_risk_v1}) becomes
\begin{align}\label{equa_bound_gain_risk_v2}
   R(h)  & = \displaystyle \int_{\mathcal{X}} \displaystyle \int_{-\infty}^{\infty} -\log \Big(f\big(u - (h(x) - h_0(x)) \big) \Big) f(u) du d P_{X_0}(x) 
 = \displaystyle \int_{\mathcal{X}}  V(h(x) - h_0(x))   d P_{X_0}(x).
\end{align}
From (\ref{equa_proof_minimizer_V1}) and (\ref{equa_bound_gain_risk_v2}), it follows that
\begin{equation}\label{equa_bound_gain_excess_risk_v1}
    R(h_0) - R(h) =   \displaystyle \int_{\mathcal{X}}  \Big(V(0) - V(h(x) - h_0(x)) \Big) d P_{X_0}(x) \leq 0.
\end{equation}
Which shows that $h_0$ achieves the minimum of the entropy risk.
\qed

\subsection{Poof of Theorem \ref{excess_risk_bound__expo_strong_mixing}}\label{Proof_excess_risk_thm1}

Consider two samples $ D_n = \{ (X_1, Y_1), \dots, (X_n, Y_n) \}$ and $ D_n' = \{ (X'_1, Y'_1), \dots, (X'_n, Y'_n) \}$ drawn from a stationary and ergodic process $Z_t = (X_t, Y_t)_{t\in \Z} \in \mathcal{X} \times \mathcal{Y}$ from model (\ref{equa_model_reg}), such that $D_n$ is independent of $D_n'$.
Let $h^{*} \in \mathcal{C}^{s}(\mathcal{X}, \mathcal{K}^{*})$ be a target predictor satisfying (\ref{best_pred_F}).
 For all $i= 1,\dots,n$, set:
\begin{equation}\label{equa_proof_decomp_excess_risk_V1}
g(h(X_i), Y_i) = -\log f(Y_i - h(X_i)) + \log f(Y_i - h^*(X_i)),
\end{equation}
for any function $h : \mathcal{X} \to \mathcal{Y}$. 
 Recall that, the MEE-based NPDNN estimator $\widehat{h}_{n,NP}$ depends on the sample $D_n$, and with expected excess risk,
\begin{equation}\label{equa_proof_excess_risk_v0}
\E[R(\widehat{h}_{n,NP}) - R(h^*) ]=  \E_{D_n} \Big[\E_{D_n'} \Big[\dfrac{1}{n} \sum_{i=1}^n g(\widehat{h}_{n,NP}, Z_i')\Big] \Big].
\end{equation}
Let $L_n, N_n, B_n, F_n, S_n$ fulfill the conditions in Theorem \ref{excess_risk_bound__expo_strong_mixing}. 
Set:
\begin{equation}\label{proof_def_H_sigma_n}
\mathcal{H}_{\sigma, n} \coloneqq \mathcal{H}_{\sigma}(L_n, N_n, B_n, F_n, S_n).
\end{equation}
%
%
%
Remark that, for any $h_1, h_2 \in \mathcal{H}_{\sigma, n}$, $\mu_1, \mu_2 \in [-1,1]$, we have
\[ \E\big[ \big| \mu_1 \big(Y_0-h_1(X_0) \big)  + \mu_2 \big(Y_0-h_2(X_0) \big)  \big|   \big] \leq  2 \E|Y_0| + 2F .  \]
So, from assumption \textbf{(A4)}, there exists 
$\mk >0$ such that,
\begin{equation}\label{appl_A4_K}
 \sup_{h_1, h_2 \in \mathcal{H}_{\sigma, n} } \sup_{\mu_1, \mu_2 \in [-1, 1]}  \E\Big[ \Big| \dfrac{ f^{(1)}(\mu_1 \big(Y_0-h_1(X_0) \big)  + \mu_2 \big(Y_0-h_2(X_0) \big)) }{ f(\mu_1 \big(Y_0-h_1(X_0) \big)  + \mu_2 \big(Y_0-h_2(X_0) \big)) } \Big| \Big] \leq \mathcal{K} .
\end{equation}
Let us consider a target neural network $h_{\mathcal{H}_{\sigma, n}} \in \mathcal{H}_{\sigma, n}$, given by:
\begin{equation*}
 h_{\mathcal{H}_{\sigma, n}} = \underset{h\in \mathcal{H}_{\sigma, n} }{\argmin}\, R(h).   
\end{equation*}
%
%
By the definition of $\widehat{h}_{n,NP}$, we get:
\begin{align}\label{equa_proof_app_ERM_V0}
 \E_{D_n}\Big[\dfrac{1}{n} \sum_{i=1}^n \log f(Y_i - \widehat{h}_{n,NP}(X_i)) \ge \E_{D_n}\Big[\dfrac{1}{n} \sum_{i=1}^n \log f(Y_i -h_{\mathcal{H}_{\sigma, n}}(X_i))\Big].
\end{align}

According to (\ref{equa_proof_excess_risk_v0}) and (\ref{equa_proof_app_ERM_V0}), we get 
\begin{align*}
\nonumber & \E[R(\widehat{h}_{n,NP}) - R(h^*) ]  = \E_{D_n} \Bigg[ \dfrac{1}{n} \sum_{i=1}^n \Big( -2g(\widehat{h}_{n,NP}, Z_i) + \E_{D_n'} \big[g(\widehat{h}_{n,NP}, Z_i')\big] + 2g(\widehat{h}_{n,NP}, Z_i) \Big) \Bigg]
\\
\nonumber & =  \E_{D_n} \Bigg[ \dfrac{1}{n} \sum_{i=1}^n \Big( -2g(\widehat{h}_{n,NP}, Z_i) + \E_{D_n'} \big[g(\widehat{h}_{n,NP}, Z_i')\big] -2\log f(Y_i - \widehat{h}_{n,NP}(X_i)) + 2\log f(Y_i - h^*(X_i))  \Big) \Bigg]  
\\
\nonumber & = \E_{D_n} \Bigg[ \dfrac{1}{n} \sum_{i=1}^n \Big(-2g(\widehat{h}_{n,NP}, Z_i) + \E_{D_n'} \big[g(\widehat{h}_{n,NP}, Z_i')\big] -2\log f(Y_i - \widehat{h}_{n,NP}(X_i))
\\
\nonumber & \hspace{2cm} + 2\log f(Y_i -h_{\mathcal{H}_{\sigma, n}}(X_i)) - 2\log f(Y_i -h_{\mathcal{H}_{\sigma, n}}(X_i))  + 2\log f(Y_i - h^*(X_i))
\Big) \Bigg] \\
 \nonumber & \leq  \E_{D_n} \Bigg[ \dfrac{1}{n} \sum_{i=1}^n \Big( -2g(\widehat{h}_{n,NP}, Z_i) + \E_{D_n'} \big[g(\widehat{h}_{n,NP}, Z_i')\big] - 2\big[\log f(Y_i - \widehat{h}_{n,NP}(X_i)) - \log f(Y_i -h_{\mathcal{H}_{\sigma, n}}(X_i)) \big] \\
\nonumber & \hspace{1cm} + 2\big[-\log f(Y_i -h_{\mathcal{H}_{\sigma, n}}(X_i)) +\log f(Y_i - h^*(X_i)) \big]
\Big) \Bigg] \\
\nonumber &  \leq \E_{D_n} \Bigg[ \dfrac{1}{n} \sum_{i=1}^n \Big(-2g(\widehat{h}_{n,NP}, Z_i) + \E_{D_n'} \big[g(\widehat{h}_{n,NP}, Z_i')\big] \Big) \Bigg]
\\
\nonumber & \hspace{1cm} -2\E_{D_n} \Bigg[\dfrac{1}{n} \sum_{i=1}^n \Big(\log f(Y_i - \widehat{h}_{n,NP}(X_i)) - \log f(Y_i -h_{\mathcal{H}_{\sigma, n}}(X_i)) \Big) \Bigg]
\\
\nonumber & \hspace{1cm} + 2\E_{D_n} \Bigg[ \dfrac{1}{n} \sum_{i=1}^n \Big(-\log f(Y_i -h_{\mathcal{H}_{\sigma, n}}(X_i)) + \log f(Y_i - h^*(X_i)) \Big) \Bigg]\\ 
\nonumber &  \leq \E_{D_n} \Bigg[ \dfrac{1}{n} \sum_{i=1}^n \Big(-2g(\widehat{h}_{n,NP}, Z_i) + \E_{D_n'} \big[g(\widehat{h}_{n,NP}, Z_i')\big] \Big)   \Bigg]
\\
\nonumber & \hspace{1cm} + 2\E_{D_n} \Bigg[ \dfrac{1}{n} \sum_{i=1}^n \Big(-\log f(Y_i -h_{\mathcal{H}_{\sigma, n}}(X_i)) +\log f(Y_i - h^*(X_i)) \Big) \Bigg]  
\end{align*} 
\begin{align} 
\label{Exp_excess_risk_v0} &  \leq \E_{D_n} \Bigg[ \dfrac{1}{n} \sum_{i=1}^n \Big(-2g(\widehat{h}_{n,NP}, Z_i) + \E_{D_n'} \big[g(\widehat{h}_{n,NP}, Z_i')\big] \Big) \Bigg] + 2\big(R( h_{\mathcal{H}_{\sigma, n}}) - R(h^*) \big)
\\
\label{equa_bound_excess_risk_v1} &  \leq A_{1, n} + 2\big(R( h_{\mathcal{H}_{\sigma, n}}) - R(h^*) \big),
\end{align}
where $A_{1, n} :=  \E_{D_n} \Bigg[ \dfrac{1}{n} \sum_{i=1}^n \Big(-2g(\widehat{h}_{n,NP}, Z_i) + \E_{D_n'} \big[g(\widehat{h}_{n,NP}, Z_i')\big] \Big) \Bigg]$.
So, it suffices to establish upper bounds of the estimation error $A_{1, n}$, and the approximation error $R( h_{\mathcal{H}_{\sigma, n}}) - R(h^*)$.

\medskip
\noindent
\textbf{Step 1: Bounding the estimation error $A_{1, n}$}. 

\noindent
Let us denote for any $h \in \mathcal{H}_{\sigma, n}$,
\[G(h, Z_i) =  \E_{D_n'} g(h, Z_i') -2g(h, Z_i).
\]
 Let $ \delta >0$.  
 Since $\|h\|_\infty \leq F_n < \infty$ for all $h \in \mathcal{H}_{\sigma, n}$, then, the $\delta$-covering number (see (\ref{epsi_covering_number})) of $\mathcal{H}_{\sigma, n}$ is finite. Set,
 \begin{equation}\label{proof_def_m_n}
 m := \mathcal{N}\Big( \mathcal{H}_{\sigma, n}, \delta  \Big).
\end{equation}    
From the Proposition 1 in \cite{ohn2019smooth}, we have,
\begin{equation}\label{proof_ing_covering_num}
\mathcal{N} \Big (\mathcal{H}_{\sigma, n}, \delta \Big)  \leq  \exp \Bigg (  2 L_n (S_n + 1) \log \left(\frac{1}{\delta} C_{\sigma} L_n (N_n + 1)(B_n \lor 
 1) \right) \Bigg),
\end{equation}
where $B_n \lor 1 = \max(B_n, 1)$.
The dependence of $m$ on $n$ is omitted for notational simplicity.
Let $h_1,\cdots,h_m \in \mathcal{H}_{\sigma,n}$ such that,
\begin{equation}\label{proof_H_sigma_n_subset_ball}
\mathcal{H}_{\sigma,n} \subset \bigcup_{j=1}^m B(h_j, \delta).
\end{equation}
$B(h_j, \delta)$ is the ball of radius $\delta$, centered at $h_j$. So, there exists a random index $j^* \in \{1,\cdots, m\}$  such that, $ \|\widehat{h}_{n,NP} - h_{j^*}\|_{\infty} \leq \delta$.
We have
\begin{align}
\label{equa_bound_A1n_V1}    \Big|G(\widehat{h}_{n,NP}, Z_i) - G(h_{j*}, Z_i) \Big| & = \Big|\E_{D_n'}[g(\widehat{h}_{n,NP}, Z_i)] - 2g(\widehat{h}_{n,NP}, Z_i) - \E_{D_n'}[g(h_{j*}, Z_i)] + 2g(h_{j*}, Z_i) \Big|
    \\
 \label{equa_bound_A1n_V2}   \Big|g(\widehat{h}_{n,NP}, Z_i) - g(h_{j*}, Z_i)\Big| & = \Big|\log f(Y_i - \widehat{h}_{n,NP}(X_i)) - \log f(Y_i - h_{j*}(X_i)) \Big|,
\end{align}
where $g$ is defined in (\ref{equa_proof_decomp_excess_risk_V1}).
According to (\ref{equa_bound_A1n_V2}), for all $i=1,\cdots,n$, it holds that,
\begin{align}
 \nonumber &\Big|\dfrac{1}{n}\sum_{i=1}^n \E_{D_n}[g(\widehat{h}_{n,NP}, Z_i)] - \dfrac{1}{n}\sum_{i=1}^n \E_{D_n}[g(h_{j^*}, Z_i)]\Big|  \leq \dfrac{1}{n}\sum_{i=1}^n \E_{D_n} \Big[ \Big|g(\widehat{h}_{n,NP}, Z_i) - g(h_{j^*}, Z_i) \Big|\Big] \\
 \nonumber & \leq \dfrac{1}{n}\sum_{i=1}^n \E_{D_n} \Big[ \Big|\log f(Y_i - \widehat{h}_{n,NP}(X_i)) - \log f(Y_i - h_{j*}(X_i))  \Big|\Big]  \\
 \label{equa_bound_A1n_V3} & \leq \dfrac{1}{n}\sum_{i=1}^n \E_{D_n} \Bigg[ \Bigg|\dfrac{f^{(1)}(U)\big( \widehat{h}_{n,NP}(X_i) - h_{j*}(X_i) \big)}{f(U)} \Bigg| \Bigg]  \leq \mathcal{K}  \delta,
\end{align}
where (\ref{equa_bound_A1n_V3}) is obtained for Mean Value Theorem, $U$ is a random variable between $Y_0 - \widehat{h}_{n,NP}(X_0)$ and $Y_0 - h_{j*}(X_0)$, and $\mathcal{K}$ is given in  (\ref{appl_A4_K}). 
In addition to (\ref{equa_bound_A1n_V1}), 
we have
\begin{align}
\nonumber &\E_{D_n} \Big[\dfrac{1}{n} \sum_{i=1}^n \Big|G(\widehat{h}_{n,NP}, Z_i) - G(h_{j^*}, Zi)\Big| \Big] \\
\nonumber &= \E_{D_n} \Big[\dfrac{1}{n} \sum_{i=1}^n \Big(\Big|\E_{D_n'}[g(\widehat{h}_{n,NP}, Z_i)] - 2g(\widehat{h}_{n,NP}, Z_i) - \E_{D_n'}[g(h_{j*}, Z_i)] + 2g(h_{j*}, Z_i)\Big| \Big) \Big]  
\\
  & \leq \E_{D_n} \Big[\dfrac{1}{n} \sum_{i=1}^n \Big(\Big|\E_{D_n'}[g(\widehat{h}_{n,NP}, Z_i)] - \E_{D_n'}[g(h_{j*}, Z_i)]\Big| + 2\Big|g(\widehat{h}_{n,NP}, Z_i) - g(h_{j*}, Z_i)\Big| \Big) \Big] \leq 3\mathcal{K}  \delta.
\end{align}
Hence, we obtain
\begin{equation}\label{euqua_bound_estimation_error}
 A_{1, n} = \E_{D_n} \Big[\dfrac{1}{n}\sum_{i= 1}^n \{ -2g(\widehat{h}_{n,NP}, Z_i)+ \E_{D_n'} g(\widehat{h}_{n,NP}, Z_i') \} \Big] = \E_{D_n}\big[\dfrac{1}{n}\sum_{i=1}^n G(\widehat{h}_{n,NP}, Z_i) \big] \leq \E_{D_n}\big[\dfrac{1}{n}\sum_{i=1}^n G(h_{j^*}, Z_i) \big] + 3 \mathcal{K} \delta.   
\end{equation}

\medskip

\noindent
Let $0 <\beta_n < 1$. Without loss of generality, we assume that $\|f\|_\infty \leq 1$ and define $T_{\beta_n}$ as the truncation at level $\beta_n$; that is,
\begin{align}\label{equa_truncation_func_f_V1}
T_{\beta_n}f = \left\{
\begin{array}{ll}
  f   & \text{ if } f\ge \beta_n  \\
  \beta_n   &  \text{otherwise}.
\end{array}
\right.
\end{align}
Let us define the function $h_{\beta_n}^*$ by
\begin{equation}\label{equa_def_hbeta_V0}
 h_{\beta_n}^*(x) = \underset{h: \|h\|_{\infty} < F_n}{\argmin} \E\big[-\log T_{\beta_n} f(Y - h(X))|X=x\big],   
\end{equation}
for each $x \in \mathcal{X}$. 
%
%
According to (\ref{equa_def_hbeta_V0}) for any $h$ satisfying $\|h\|_{\infty} < F_n$, we get
\begin{equation}\label{equa_bound_def_hbeta_v1}
\E\big[-\log T_{\beta_n} f( h_{\beta_n}^*(X_i) - Y_i)\big] \leq \E\big[-\log T_{\beta_n} f(Y_0 - h(X_0))\big].    
\end{equation}
%
%
For any $h\in \mathcal{H}_{\sigma, n}$, let $g_{\beta_n}(h, Z_0) = -\log T_{\beta_n} f(Y_0 - h(X_0)) + \log T_{\beta_n} f(Y_0 - h_{\beta_n}^*(X_0) )$. Then, we get
\begin{align}\label{equa_bound_func_gv1}
 \nonumber   &\E[g(h, Z_0)] \\
 \nonumber &  = \E\big[ g_{\beta_n}(h, Z_0)  + g(h, Z_0) - g_{\beta_n}(h, Z_0) + \log T_{\beta_n} f(Y_0 - h^*(X_0) ) - \log T_{\beta_n} f(Y_0 - h^*(X_0) ) \big]
 \\
 \nonumber & =  \E\big[ g_{\beta_n}(h, Z_0)\big] + \E\Big[-\log f(Y_0 - h(X_0)) + \log f(Y_0 - h^*(X_0) )  + \log T_{\beta_n} f(Y_0 - h(X_0)) - \log T_{\beta_n} f(Y_0 - h_{\beta_n}^*(X_0) ) 
 \\
 \nonumber & \hspace{2cm}  + \log T_{\beta_n} f(Y_0 - h^*(X_0) ) - \log T_{\beta_n} f(Y_0 - h^*(X_0) ) \Big]
 \\
 \nonumber & =  \E\big[ g_{\beta_n}(h, Z_0)\big] + \E\big[ \log T_{\beta_n} f(Y_0 - h(X_0)) - \log f(Y_0 - h(X_0))\big]  \\
 \nonumber & \hspace{1cm} +  \E\big[  \log T_{\beta_n} f(Y_0 - h^*(X_0) )  - \log T_{\beta_n} f(Y_0 - h_{\beta_n}^*(X_0) ) \big] + \E\big[ \log f(Y_0 - h^*(X_0) ) - \log T_{\beta_n} f(Y_0 - h^*(X_0) )\big] 
 \\
 & =  \E\big[ g_{\beta_n}(h, Z_0)\big] + \E\big[ \log T_{\beta_n} f(Y_0 - h(X_0)) - \log f(Y_0 - h(X_0))\big] + \E\big[ \log f(Y_0 - h^*(X_0) ) - \log T_{\beta_n} f(Y_0 - h^*(X_0) )\big],
\end{align}
where (\ref{equa_bound_func_gv1}) follows from (\ref{equa_bound_def_hbeta_v1}).

Recall that $\|h\|_{\infty} < F_n=F$ and for a positive constant $C_0 > 0$ we have $\|h_0\|_{\infty} < C_0$. Thus, we have 
\[\big|h_0(X_0) - h(X_0) \big| < C_0 + F.
\]
Then, from assumption \textbf{(A5)} with $C = C_0 + F$, there exists $p = p(C) \ge 1$ such that
\begin{equation}\label{equa_bound_assump_A5_V0} 
\E\Big[\Big|\log f(\epsilon_0 + h_0(X_0) - h(X_0)) \big) \ind_{\{f(\epsilon_0 + h_0(X_0) - h(X_0)) \leq \beta_n \}} \Big| \Big] \lesssim \dfrac{\log n}{n}.
\end{equation}
where $\beta_n = n^{-p}$ for all $n$ large enough.
One can easily see that
\begin{align}\label{equa_bound_func_gv2_v2}
    \nonumber & \E\Big[\Big|\log T_{\beta_n} f(Y_0 - h(X_0)) - \log f(Y_0 - h(X_0)) \Big| \Big] \\
    \nonumber & = \E\Big[\Big|\log f(Y_0 - h(X_0)) \ind_{\{f(Y_0 - h(X_0)) \ge \beta_n\}} + \log \beta_n \ind_{\{f(Y_0 - h(X_0)) \leq \beta_n\}}  - \log f(Y_0 - h(X_0)) \Big| \Big]  \\
 \nonumber  & =  \E\Big[\Big|\big( \log \beta_n - \log f(Y_0 - h(X_0)) \big) \ind_{\{f(Y_0 - h(X_0)) \leq \beta_n\}} \Big| \Big]\\
\nonumber & \leq \E\Big[\Big|\log \beta_n \ind_{\{f(Y_0 - h(X_0)) \leq \beta_n\}} \Big| +   \Big| \log f(Y_0 - h(X_0)) \big) \ind_{\{f(Y_0 - h(X_0))\leq \beta_n\}} \Big| \Big] \\
 \nonumber & \leq \E\Big[\Big| \log f(Y_0 - h(X_0)) \big) \ind_{\{f(Y_0 - h(X_0)) \leq \beta_n\}} \Big| ~ \text{ since } \log \beta_n < 0 \Big] 
 \\
  & \leq  \Big[\Big| \log f(\epsilon_0 + h_0(X_0)- h(X_0)) \big) \ind_{\{f(\epsilon_0 + h_0(X_0) - h(X_0)) \leq \beta_n\}} \Big| \Big]  \leq\dfrac{\log n}{n},
\end{align}
where (\ref{equa_bound_func_gv2_v2}) follows from (\ref{equa_bound_assump_A5_V0}). 

\medskip

Recall that $\|h^*\|_{\infty} < \mathcal{K}^*$ and for a positive constant $C_0 > 0$ we have $\|h_0\|_{\infty} < C_0$. Thus, we have 
\[\big|h_0(X_0) - h^*(X_0) \big| < C_0 + \mathcal{K}^*.
\]
Then, from assumption \textbf{(A5)} with $C = C_0 + \mathcal{K}^*$, there exists $p = p(C) \ge 1$ such that
\begin{equation}\label{equa_bound_assump_A5_V2} 
\E\Big[\Big|\log f(\epsilon_0 + h_0(X_0) - h^*(X_0)) \big) \ind_{\{f(\epsilon_0 + h_0(X_0) - h^*(X_0)) \leq \beta_n \}} \Big| \Big] \leq \dfrac{\log n}{n}.
\end{equation}
where $\beta_n = n^{-p}$ for all $n$ large enough.
From assumption \textbf{(A5)} and using same argument as above, we have

\begin{align}\label{equa_bound_func_gv3}
    \E\Big[\Big|\log f(Y_0 - h^*(X_0) ) - \log T_{\beta_n} f(Y_0 - h^*(X_0) )\Big| \Big] \lesssim \dfrac{\log n}{n}.  
\end{align}
According to (\ref{equa_bound_func_gv1}), (\ref{equa_bound_func_gv2_v2}) and (\ref{equa_bound_func_gv3}), we obtain
\begin{align}\label{equa_bound_func_gv4}
 \E[g(h, Z_0)] \lesssim \E\big[ g_{\beta_n}(h, Z_0)\big] + \dfrac{\log n}{n}.
\end{align}
Therefore,
\begin{align}
 \nonumber   &\E[g_{\beta_n}(h, Z_0)] \\
 \nonumber &  = \E\big[ g(h, Z_0)  + g_{\beta_n}(h, Z_0) - g(h, Z_0) + \log f(Y_0 - h_{\beta_n}^*(X_0) ) - \log f(Y_0 - h_{\beta_n}^*(X_0) ) \big]
 \\
 \nonumber & =  \E\big[ g(h, Z_0)\big] + \E\Big[-\log T_{\beta_n} f(Y_0 - h(X_0)) + \log T_{\beta_n} f(Y_0 - h_{\beta_n}^*(X_0) )  + \log f(Y_0 - h(X_0)) - \log f(Y_0 - h^*(X_0) ) 
 \\
 \nonumber & \hspace{2cm}  + \log f(Y_0 - h_{\beta_n}^*(X_0) ) - \log  f(Y_0 - h_{\beta_n}^*(X_0) ) \Big]
 \\
 \nonumber & =  \E\big[ g(h, Z_0)\big] + \E\big[ \log f(Y_0 - h(X_0)) - \log T_{\beta_n} f(Y_0 - h(X_0))\big]  \\
 \nonumber & \hspace{1cm} +  \E\big[  \log f(Y_0 - h_{\beta_n}^*(X_0) )  - \log f(Y_0 - h^*(X_0) ) \big] + \E\big[ \log T_{\beta_n} f(Y_0 - h_{\beta_n}^*(X_0) ) - \log f(Y_0 - h_{\beta_n}^*(X_0) )\big] 
 \\
\label{equa_bound_func_gbetanv1} & \leq   \E\big[ g(h, Z_0)\big] + \E\big[ \log f(Y_0 - h(X_0)) - \log T_{\beta_n} f(Y_0 - h(X_0))\big] + \E\big[ \log T_{\beta_n} f(Y_0 - h_{\beta_n}^*(X_0) ) - \log f(Y_0 - h_{\beta_n}^*(X_0) )\big]
 \\
\label{equa_bound_func_gbetanv2} & \lesssim  \E\big[ g(h, Z_0)\big] + \dfrac{\log n}{n},
\end{align}
where (\ref{equa_bound_func_gbetanv1}) follows from (\ref{equa_def_hbeta_V0}), and (\ref{equa_bound_func_gbetanv2}) holds by analogy with  (\ref{equa_bound_func_gv2_v2}) and (\ref{equa_bound_func_gv3}).

\medskip

\noindent
So, for any $h\in \mathcal{H}_{\sigma, n}$, let $G_{\beta_n}(h, Z_0) = \E_{D_n}[g_{\beta_n}(h, Z_0)] - 2g_{\beta_n}(h, Z_0).
$ 
Thus, we get
\begin{align}\label{equa_bound_func_Gbetanv1}
    \nonumber G(h, Z_0) - G_{\beta_n}(h, Z_0) & = \Big( \E_{D_n}[g(h, Z_0)] - 2g(h, Z_0) - \E_{D_n}[g_{\beta_n}(h, Z_0)] + 2g_{\beta_n}(h, Z_0)\Big) \\
&= \Big( \E_{D_n}[g(h, Z_0)] -\E_{D_n}[g_{\beta_n}(h, Z_0)] + 2\big(g_{\beta_n}(h, Z_0) - g(h, Z_0) \big) \Big),
\end{align}
According to (\ref{equa_bound_func_gv4}), (\ref{equa_bound_func_gbetanv2}), and (\ref{equa_bound_func_Gbetanv1}), we have
\begin{align}\label{equa_bound_func_Gbetanv2}
    \nonumber \E_{D_n}\Big[\dfrac{1}{n} \sum_{i=1}^n \Big(G(h, Z_i) - G_{\beta_n}(h, Z_i) \Big)\Big] & = \E_{D_n}\Big[\dfrac{1}{n} \sum_{i=1}^n \Big(\E_{D_n}[g(h, Z_i)] -\E_{D_n}[g_{\beta_n}(h, Z_i)] + 2\big(g_{\beta_n}(h, Z_i) - g(h, Z_i) \big) \Big) \Big] \\
    & \leq 3\dfrac{\log n}{n}.
\end{align}
It follows from (\ref{equa_bound_func_Gbetanv2}) that
\begin{align}\label{equa_bound_func_Gbetanv3}
 \nonumber   \E_{D_n}\Big[\dfrac{1}{n} \sum_{i=1}^n \Big( G(h, Z_i) \Big) \Big] & =  \E_{D_n}\Big[\dfrac{1}{n} \sum_{i=1}^n \Big(G(h, Z_i) - G_{\beta_n}(h, Z_i) + G_{\beta_n}(h, Z_i) \Big)\Big] 
 \\
& \leq \E_{D_n}\Big[\dfrac{1}{n} \sum_{i=1}^n G_{\beta_n}(h, Z_0) \Big] + 3\dfrac{\log n}{n}.
\end{align}
In addition to (\ref{equa_bound_func_Gbetanv3}), we have
\begin{align}\label{equa_bound_func_Gbetanv4}
\nonumber A_{1, n} & =\E_{D_n}\Big[\dfrac{1}{n}\sum_{i=1}^n G(\widehat{h}_{n,NP}, Z_i) \Big] \leq \E_{D_n}\Big[\dfrac{1}{n}\sum_{i=1}^n G(h_{j^*}, Z_i) \Big] + 3 \mathcal{K} \delta \\
& \leq \E_{D_n}\Big[\dfrac{1}{n} \sum_{i=1}^n G_{\beta_n}(h_{j^*}, Z_0) \Big] + 3\dfrac{\log n}{n} + 3 \mathcal{K} \delta. 
\end{align}
In the sequel, we derive an upper bound for $\E_{D_n}\Big[\dfrac{1}{n} \sum_{i=1}^n G_{\beta_n}(h_{j^*}, Z_i) \Big]$. We have
\begin{align}\label{equa_bound_abs_gbetan_V0}
\nonumber   \big|g_{\beta_n}(h, Z_0) \big| & =  \Big|\log T_{\beta_n} f(Y_0 - h_{\beta_n}^*(X_0) ) - \log T_{\beta_n} f(Y_0 - h(X_0)) \Big| \\
    & \leq \Big|\log T_{\beta_n} f(Y_0 - h_{\beta_n}^*(X_0) )   \Big| + \Big|\log T_{\beta_n} f(Y_0 - h(X_0)) \Big|   \leq 2|\log \beta_n|,
\end{align}
where (\ref{equa_bound_abs_gbetan_V0}) follows from (\ref{equa_truncation_func_f_V1}). According to (\ref{equa_bound_abs_gbetan_V0}), we have
\begin{equation}\label{equa_bound_var_gbetan_V0}
    \var \big(g_{\beta_n}(h, Z_0) \big) \leq \E\big[g_{\beta_n}^2(h, Z_0)\big] \leq 2|\log \beta_n| \E\big[g_{\beta_n}(h, Z_0)\big]  
\end{equation}
and
\begin{equation}\label{equa_bound_abs_moy_gbetan}
    \Big|\E\big[g_{\beta_n}(h, Z_0)\big] - g_{\beta_n}(h, Z_0)\Big| \leq 4|\log \beta_n|. 
\end{equation}
%
Let $\rho > 0$. According to (\ref{equa_bound_var_gbetan_V0}) and by using the exponential inequality provided in Subsection 3.3.3 of \cite{hang2016learning}, which builds upon the result from \cite{merlevede2009bernstein},
we have for all $j=1,\cdots,m$ (where $m$ is defined in (\ref{proof_def_m_n})) :
\begin{align*} 
\nonumber &  P\Bigg\{\dfrac{1}{n} \sum_{i=1}^n  G_{\beta_n}(h_j, Z_i) > \rho \Bigg\}  =  P\Bigg\{ \E_{D_n}\big[g_{\beta_n}(h_j, Z_1)\big] - \dfrac{2}{n} \sum_{i=1}^n g_{\beta_n}(h_j, Z_i) > \rho \Bigg\}
\\
 \nonumber & =  P\Bigg\{ 2\E_{D_n}\big[g_{\beta_n}(h_j, Z_1)\big] - \dfrac{2}{n} \sum_{i=1}^n g_{\beta_n}(h_j, Z_i) > \rho + \E_{D_n}\big[g_{\beta_n}(h_j, Z_i)\big] \Bigg\}
 \\
 \nonumber & =  P\Bigg\{\E_{D_n}\big[g_{\beta_n}(h_j, Z_1)\big] - \dfrac{1}{n} \sum_{i=1}^n g_{\beta_n}(h_j, Z_i) > \dfrac{1}{2} \big(\rho + \E_{D_n}\big[g_{\beta_n}(h_j, Z_i)\big] \big) \Bigg\}
  \\
  \nonumber & =   P\Bigg\{\dfrac{1}{n} \sum_{i=1}^n \Big(\E_{D_n}\big[g_{\beta_n}(h, Z_1)\big] - g_{\beta_n}(h, Z_i) \Big)  > \dfrac{1}{2} \big(\rho + \E_{D_n}\big[g_{\beta_n}(h_j, Z_i)\big] \big) \Bigg\}
\end{align*}
\begin{align*}
\nonumber& \leq C_2 \exp \Bigg[ - \dfrac{C_3\Big( \dfrac{\rho}{2} + \dfrac{\E_{D_n}\big[g_{\beta_n}(h_j, Z_i)\big]}{2} \Big)^2 \Big(n/\big(\log n\big)^2\Big)}{6 \log |\beta_n| \E_{D_n}\big[g_{\beta_n}(h, Z_1)\big] + 4\log |\beta_n| \Big(\dfrac{\rho}{2} + \dfrac{\E_{D_n}\big[g_{\beta_n}(h_j, Z_1)\big]}{2} \Big)}\Bigg]
\\
\nonumber & \leq C_2 \exp \Bigg[ - \dfrac{C_3\Big( \dfrac{\rho}{2} + \dfrac{\E_{D_n}\big[g_{\beta_n}(h_j, Z_i)\big]}{2} \Big)^2 \Big(n/\big(\log n\big)^2\Big)}{6 \log |\beta_n|\Big(\dfrac{\rho}{2} + \dfrac{\E_{D_n}\big[g_{\beta_n}(h_j, Z_1)\big]}{2} \Big) + 4\log |\beta_n| \Big(\dfrac{\rho}{2} + \dfrac{\E_{D_n}\big[g_{\beta_n}(h_j, Z_1)\big]}{2} \Big)}\Bigg]\\
\nonumber & \leq C_2 \exp \Bigg[ - \dfrac{C_3\Big( \dfrac{\rho}{2} + \dfrac{\E_{D_n}\big[g_{\beta_n}(h_j, Z_i)\big]}{2} \Big)^2 \Big(n/\big(\log n\big)^2\Big)}{10 \log |\beta_n|\Big(\dfrac{\rho}{2} + \dfrac{\E_{D_n}\big[g_{\beta_n}(h_j, Z_1)\big]}{2} \Big) }\Bigg]
\\
\nonumber & \leq C_2 \exp \Bigg[ - \dfrac{C_3\Big( \dfrac{\rho}{2} + \dfrac{\E_{D_n}\big[g_{\beta_n}(h_j, Z_i)\big]}{2} \Big) \Big(n/\big(\log n\big)^2\Big)}{10 \log |\beta_n| }\Bigg]  \leq C_2 \exp \Bigg[ - \dfrac{C_3 \rho \big(n/\big(\log n\big)^2\big)}{10 \log |\beta_n| }\Bigg],
\end{align*}
for some positives constans $C_2, C_3 > 0$.
Set, $\delta_n = \frac{\big(\log n\big)^2}{n}$, see (\ref{proof_def_m_n}). That is $m = \mathcal{N}\Big( \mathcal{H}_{\sigma, n},  \big(\log n\big)^2/ n\Big)$. In addition to (\ref{proof_ing_covering_num}) we have,
\begin{align}\label{Proof_bernstein_Gbeta_V1}
\nonumber P \Big\{\dfrac{1}{n} \sum_{i = 1}^n G_{\beta_n}(h_{j}, Z_i) > \rho \Big\}  & \leq   P \Bigg\{ \bigcup_{j=1}^m \Big( \dfrac{1}{n} \sum_{i = 1}^n G_{\beta_n}(h_{j}, Z_i) > \rho \Big) \Bigg\} \leq \sum_{j=1}^m  P \Big\{\dfrac{1}{n} \sum_{i = 1}^n G_{\beta_n}(h_{j}, Z_i) > \rho \Big\}  
\\
\nonumber & \leq \mathcal{N}\Big( \mathcal{H}_{\sigma, n},  \frac{\big(\log n\big)^2}{n} \Big) \cdot C_2 \exp \Bigg[ - \dfrac{C_3 \rho \big(n/\big(\log n\big)^2\big)}{10 \log |\beta_n| }\Bigg]
\\
& \leq C_2\exp \Bigg[2L_n(S_n + 1)\log\Big(n C_2{\sigma}L_n(N_n + 1)(B_n \lor 1)/ (\log n)^2\Big) - \dfrac{C_3 \rho \big(n/\big(\log n\big)^2\big)}{10 \log |\beta_n| } \Bigg]. 
\end{align}
Recall $\beta_n = \dfrac{1}{n^p}$ with $p \ge 1$. Thus (\ref{Proof_bernstein_Gbeta_V1}) becomes
\begin{align}\label{Proof_bernstein_Gbeta_V2}
\nonumber P \Big\{\dfrac{1}{n} \sum_{i = 1}^n G_{\beta_n}(h_{j}, Z_i) > \rho \Big\} \leq C_2 \exp \Bigg[2L_n(S_n + 1)\log\Big(n C_2{\sigma}L_n(N_n + 1)(B_n \lor 1)/ (\log n)^2\Big) - \dfrac{C_3 \rho n}{10 p \big(\log n\big)^3} \Bigg]. 
\end{align}
For $\alpha_n > 0$, we have
\begin{align}
\nonumber  &\E\Big[\dfrac{1}{n}\sum_{i = 1}^n G_{\beta_n}(h_{j}, Z_i) \Big]  \leq \alpha_n + \displaystyle\int_{\alpha_n}^{\infty} P\Big\{\dfrac{1}{n}\sum_{i = 1}^n G_{\beta_n}(h_{j}, Z_i) > \rho \Big\} d\varepsilon  
\\
 & \leq \alpha_n + C_2\exp \Bigg[2L_n(S_n + 1)\log\Big(n C_2{\sigma}L_n(N_n + 1)(B_n \lor 1)/ (\log n)^2\Big)  \Bigg] \times \displaystyle\int_{\alpha_n}^{\infty} \exp \Bigg( - \dfrac{C_3 \rho n}{10 p \big(\log n\big)^3} \Bigg) d\rho.    
\end{align}
Set 
\begin{equation}\label{equa_alpha_n}
\alpha_n \coloneqq \dfrac{(\log n)^{\nu}}{ n^{\frac{\kappa s}{\kappa s + d} } },  
\end{equation} 
By going same step as in the proof of Theorem 3.1 in \cite{kengne2025general}, for some $\nu > 6$, we have,
\begin{align}\label{proof_E_G_beta_log_nu_n}
\nonumber &\E\Big[ \dfrac{1}{n}\sum_{i = 1}^n G_{\beta_n}(h_{j}, Z_i) \Big] \\
& \leq \dfrac{(\log n)^{\nu}}{ n^{\frac{\kappa s}{\kappa s + d} } } + \frac{10 p C_2 \big(\log n\big)^3}{C_3 n} \exp\Bigg[2L_n(S_n + 1)\log \Big(n C_{\sigma}L_n(N_n + 1)(B_n \lor 1)/(\log n)^2 \Big) - \dfrac{C_3 n^{\frac{d}{\kappa s + d} } \big(\log n\big)^{\nu}}{10 p \big(\log n\big)^3} \Bigg].
\end{align}
Recall the assumptions:
$
L_n = \dfrac{s L_0}{\kappa s + d} \log n, N_n = N_0 n^{\frac{d}{\kappa s + d}}, S_n = \frac{s S_0}{\kappa s + d} n^{\frac{d}{\kappa s + d}}\log n, 
$
$ B_n = B_0 n^{\frac{4(d + s)}{\kappa s + d}}$,
with $L_0, N_0, B_0, S_0 >0$. 
Hence,
\begin{align}
\nonumber & \E\Big[\dfrac{1}{n}\sum_{i = 1}^n G_{\beta_n}(h_{j}, Z_i) \Big] \\
& \leq  \dfrac{(\log n)^{\nu}}{ n^{\frac{\kappa s}{\kappa s + d} } } + \frac{10 p C_2 \big(\log n\big)^3}{C_3 n}  \exp\Bigg[2\dfrac{s L_0}{\kappa s + d} \log n \Big(\frac{s S_0}{\kappa s + d} n^{\frac{d}{\kappa s + d}}\log (n) + 1 \Big) \\ 
\nonumber & \times \log\Big(n C_{\sigma} \dfrac{s L_0}{\kappa s + d} \log n \big(N_0 n^{\frac{d}{\kappa s + d}} + 1 \big) \big(B_0 n^{\frac{4(d + s)}{\kappa s + d}} \lor 1 \big) /(\log n)^2 \Big) - \dfrac{C_3 n^{\frac{d}{\kappa s + d} } \big(\log n\big)^{\nu}}{10 p \big(\log n\big)^3}  \Bigg] 
\\
%
\nonumber &  \leq \dfrac{(\log n)^{\nu}}{ n^{\frac{\kappa s}{\kappa s + d} } } + \frac{10 p C_2 \big(\log n\big)^3}{C_3 n} \exp\Bigg[\dfrac{C_3 n^{\frac{d}{\kappa s + d} } \big(\log n\big)^{\nu}}{10 p \big(\log n\big)^3} \Bigg( \dfrac{20 p \big(\log n\big)^3}{C_3 (\log n)^{\nu} } \dfrac{s L_0}{\kappa s + d} \log n \Big(\frac{s S_0}{\kappa s + d} \log (n) + 1 \Big) \\ 
 & \times \log\Big(n C_{\sigma} \dfrac{s L_0}{\kappa s + d} \log n \big(N_0 n^{\frac{d}{\kappa s + d}} + 1 \big) \big(B_0 n^{\frac{4(d + s)}{\kappa s + d}} \lor 1 \big) /(\log n)^2 \Big) - 1 \Bigg) \Bigg].
\end{align}

\medskip

Since $\nu > 6$, we have
\begin{align*}
&\dfrac{20 p \big(\log n\big)^3}{C_3 (\log n)^{\nu} } \dfrac{s L_0}{\kappa s + d} \log n \Big(\frac{s S_0}{\kappa s + d} \log (n) + 1 \Big) \\ 
\nonumber & \times \log\Big(n C_{\sigma} \dfrac{s L_0}{\kappa s + d} \log n \big(N_0 n^{\frac{d}{\kappa s + d}} + 1 \big) \big(B_0 n^{\frac{4(d + s)}{\kappa s + d}} \lor 1 \big) /(\log n)^2 \Big) \limiten 0.
\end{align*}
Hence, there exists $n_0 = n_0(\kappa, L_0, N_0, B_0, S_0, C_\sigma, s, d,)$ such that, for any $n \ge n_0$, we have
\begin{align}\label{equa_limite}
\nonumber & \dfrac{20 p \big(\log n\big)^3}{C_3 (\log n)^{\nu} } \dfrac{s L_0}{\kappa s + d} \log n \Big(\frac{s S_0}{\kappa s + d} \log (n) + 1 \Big) \\ 
& \times \log\Big(n C_{\sigma} \dfrac{s L_0}{\kappa s + d} \log n \big(N_0 n^{\frac{d}{\kappa s + d}} + 1 \big) \big(B_0 n^{\frac{4(d + s)}{\kappa s + d}} \lor 1 \big) /(\log n)^2 \Big) < \frac{1}{2}.
\end{align}
According to (\ref{equa_limite}), we have
\begin{align}\label{equa2_bound_G_beta_n_j_star}
\E\Big[\dfrac{1}{n}\sum_{i = 1}^nG(h_{j}, Z_i) \Big] & \leq \dfrac{(\log n)^{\nu}}{ n^{\frac{\kappa s}{\kappa s + d} } } + \frac{10 p C_2 \big(\log n\big)^3}{C_3 n} \exp\Bigg[-\dfrac{C_3 n^{\frac{d}{\kappa s + d} } \big(\log n\big)^{\nu}}{20 p \big(\log n\big)^3} \Bigg].
\end{align}
%
%
So, (\ref{euqua_bound_estimation_error}), and (\ref{equa2_bound_G_beta_n_j_star}) give,
\begin{align*}\label{equa_bound_stochastic_error}
A_{1, n} = \E_{D_n}  \Big[\dfrac{1}{n}\sum_{i= 1}^n \{ -2g(\widehat{h}_{n,NP}, Z_i)+ \E_{D_n'} g(\widehat{h}_{n,NP}, Z_i') \} \Big] & \leq  \dfrac{(\log n)^{\nu}}{ n^{\frac{\kappa s}{\kappa s + d} } } + \frac{10 p C_2 \big(\log n\big)^3}{C_3 n} + 3\mathcal{K}\delta.   
\end{align*} 
Since we have set $\delta_n = \dfrac{(\log n)^2}{n}$, it holds that,
\begin{equation} \label{equa_bound_stochastic_error_v1}
A_{1, n} =  \E_{D_n}  \Big[\dfrac{1}{n}\sum_{i= 1}^n \{ -2g(\widehat{h}_{n,NP}, Z_i)+ \E_{D_n'} g(\widehat{h}_{n,NP}, Z_i') \} \Big]  \leq   \dfrac{(\log n)^{\nu}}{ n^{\frac{\kappa s}{\kappa s + d} } } + \frac{10 p C_2 \big(\log n\big)^3}{C_3 n}  + \frac{3\mathcal{K}(\log n)^2}{n}.    
\end{equation} 

\medskip

\noindent
\textbf{Step 2 : Bounding the second term in the right-hand side of (\ref{equa_bound_excess_risk_v1})}. 

\medskip

\noindent
Let $L_n, N_n, B_n, S_n$ and $F_n = F > 0$ fulfill the conditions in Theorem \ref{excess_risk_bound__expo_strong_mixing}.
Recall that $\mathcal{H}_{\sigma, n} = \mathcal{H}_{\sigma}(L_n, N_n, B_n, F_n, S_n)$ see (\ref{proof_def_H_sigma_n})
and set
$ \widetilde{\mathcal{H}}_{\sigma, n} := \{h \in  \mathcal{H}_{\sigma,n}, ~ \|h - h^*\|_{\infty, \mx}  \leq \epsilon_n\} $.
By going along similar lines as in Step 2 in the proof of Theorem 3.1 in \cite{kengne2025general} with $\epsilon_n = \dfrac{ 1}{ n^{\frac{s}{\kappa s + d}} }$, we have
\begin{equation}\label{approxi_error_v1}
R(h_{\mathcal{H}_{\sigma, n}}) - R(h^{*})
\leq \dfrac{ \mk_0}{n^{\frac{\kappa s}{\kappa s + d}} } .
\end{equation}

\medskip

\noindent
\textbf{Step 3 : Expected excess risk bound}.
\medskip

\noindent
According to (\ref{equa_bound_excess_risk_v1}), (\ref{equa_bound_stochastic_error_v1}), and (\ref{approxi_error_v1}), we have,
\begin{align*}
\E[R(\widehat{h}_{n,NP}) - R(h^{*})]  
& \leq  \dfrac{(\log n)^{\nu}}{ n^{\frac{\kappa s}{\kappa s + d} } } + \frac{10 p C_2 \big(\log n\big)^3}{C_3 n}  + \frac{3\mathcal{K}(\log n)^2}{n} + \dfrac{ \mk_0}{n^{\frac{\kappa s}{\kappa s + d}} }.
\\
& \lesssim \dfrac{(\log n)^{\nu} + \mk_0}{ n^{\frac{\kappa s}{\kappa s + d} } } +  \frac{ \big(\log n\big)^3 + \big(\log n\big)^2}{n} \lesssim \dfrac{(\log n)^{\nu} }{ n^{\frac{\kappa s}{\kappa s + d} } } .
\end{align*}
%
%
Which completes the proof.
\qed

\subsection{Proof of Theorem \ref{excess_risk_bound__expo_strong_mixing_thm1_v2}}
Recall the following decomposition of the expected excess risk in (\ref{Exp_excess_risk_v0}):
\begin{align}\label{Exp_excess_risk_v3}
\E[R(\widehat{h}_n) - R(h^{*})] \leq  \E_{D_n} \Big[\dfrac{1}{n}\sum_{i= 1}^n \{ -2g(\widehat{h}_n, Z_i)+ \E_{D_n'} g(\widehat{h}_n, Z_i') \} \Big] + 2 [R(h_{\mathcal{H}_{\sigma, n}}) - R(h^{*})]. 
\end{align}
Let $\nu >6$.
Using similar arguments as in \textbf{Step 1} in the proof of Theorem \ref{excess_risk_bound__expo_strong_mixing}, and by setting $ \alpha_n \coloneqq (\log n)^{\nu}\phi_n$ as in (\ref{equa_alpha_n}), and $\phi_n$ defined in (\ref{def_phi_nalpha}), we obtain a bound for the first term in the right-hand side of (\ref{Exp_excess_risk_v3}), given by: 
\begin{equation} \label{equa_bound_stochastic_error_v3}
 \E_{D_n}  \Big[\dfrac{1}{n}\sum_{i= 1}^n \{ -2g(\widehat{h}_n, Z_i)+ \E_{D_n'} g(\widehat{h}_n, Z_i') \} \Big]  \lesssim (\log n)^{\nu}\phi_n + \frac{(\log n)^3 + (\log n)^2}{n}.    
\end{equation} 
Let us derive a bound of the second term in the right-hand side of (\ref{Exp_excess_risk_v3}).
Let $L_n, N_n, B_n, F_n = F$ satisfying the conditions in Theorem \ref{excess_risk_bound__expo_strong_mixing_thm1_v2}.
Set $\mathcal{H}_{\sigma, n} \coloneqq \mathcal{H}_{\sigma}(L_n, N_n,, B_n, F_n, S_n)$, and consider the class of composition structured functions $\mathcal{G}(q, \bold{d}, \bold{t}, \boldsymbol{\beta}, \mk^{*})$ with $h ^*
\in \mathcal{G}(q, \bold{d}, \bold{t}, \boldsymbol{\beta}, \mk^{*})$.
By going along similar lines as in the proof Theorem 3.2 in \cite{kengne2025general}, we get
\begin{align}\label{excess_risk_bound_compo_holder_funct}
\nonumber R(h_{\mathcal{H}_{\sigma, n}}) - R(h^{*}) 
 & \lesssim \phi_n^{\kappa /2}.
\end{align}
According to (\ref{Exp_excess_risk_v3}), (\ref{equa_bound_stochastic_error_v3}), we obtain
\begin{align*}
\E[R(\widehat{h}_n) - R(h^{*})] \lesssim  (\log n)^{\nu}\phi_n + \frac{(\log n)^3 + (\log n)^2}{n} + \phi_n ^{\kappa /2}  \lesssim \Big(\phi_{n}^{\kappa/2} \lor \phi_n \Big) (\log n)^{\nu}. 
\end{align*}
Thus, the theorem follows.
\qed

\subsection{Poof of Theorem \ref{excess_risk_bound__expo_strong_mixing_thm2}}
Recall that the sample $ D_n = \{ (X_1, Y_1), \dots, (X_n, Y_n) \}$ is a trajectory of a stationary and ergodic process $Z_t = (X_t, Y_t)_{t\in \Z}$ satisfying (\ref{equa_model_reg}).
Consider another sample $ D_n' \coloneqq  \{ (X_1', Y_1'), \dots, (X_n', Y_n') \}$ independent of $D_n$ and generated from the model (\ref{equa_model_reg}). 
Let $h^*$ given in (\ref{best_pred_F}). 
For any  $h : \mathcal{X} \to \mathcal{Y}$ and $i= 1,\dots,n$, set:  
\begin{equation}\label{equa_proof_decomp_excess_risk_V1}
g(h(X_i), Y_i) = -\log f(Y_i - h(X_i)) + \log f(Y_i - h^*(X_i)).
\end{equation}
So, we have
\begin{equation}\label{equa_proof_excess_risk_v1}
\E[R(\widehat{h}_{n,SP}) - R(h^*) ]=  \E_{D_n} \Big[\E_{D_n'} \Big[\dfrac{1}{n} \sum_{i=1}^n g(\widehat{h}_{n,SP}, Z_i')\Big] \Big].
\end{equation}
Let $L_n, N_n, B_n, F$ fulfill the conditions in Theorem \ref{excess_risk_bound__expo_strong_mixing_thm2}. Set:
\[\mathcal{H}_{\sigma, n} \coloneqq \mathcal{H}_{\sigma}(L_n, N_n, B_n, F).
\]
Let $h_{\mathcal{H}_{\sigma, n}}, \widetilde{ h}_{\mathcal{H}_{\sigma, n}} \in \mathcal{H}_{\sigma, n}$, such that
\begin{equation*}
 h_{\mathcal{H}_{\sigma, n}} \in \underset{h\in \mathcal{H}_{\sigma, n} }{\argmin} R(h) ~\text{ and } ~ \widetilde{h}_{\mathcal{H}_{\sigma, n}} \in \underset{h\in \mathcal{H}_{\sigma, n} }{\argmin} \Big[R(h) + J_n(h) \Big].  \end{equation*}
%
%
By the definition of the empirical risk minimizer, we get:

\begin{align}\label{equa_proof_app_ERM}
 \E_{D_n}\Big[\dfrac{1}{n} \sum_{i=1}^n \log f(Y_i - \widehat{h}_{n,SP}(X_i)) - J_n(\widehat{h}_{n,SP}) \ge \E_{D_n}\Big[\dfrac{1}{n} \sum_{i=1}^n \log f(Y_i -\widetilde{ h}_{\mathcal{H}_{\sigma, n}}(X_i)) -J_n \big(\widetilde{ h}_{\mathcal{H}_{\sigma, n}} \big)\Big].
\end{align}
According to (\ref{equa_proof_excess_risk_v1}) and (\ref{equa_proof_app_ERM}), we get 
\begin{align*}
\nonumber & \E[R(\widehat{h}_{n,SP}) - R(h^*) ] \\
\nonumber & = \E_{D_n} \Bigg[ \dfrac{1}{n} \sum_{i=1}^n \Big( -2g(\widehat{h}_{n,SP}, Z_i) + \E_{D_n'} \big[g(\widehat{h}_{n,SP}, Z_i')\big] + 2g(\widehat{h}_{n,SP}, Z_i) \Big) \Bigg]
\\
\nonumber & =  \E_{D_n} \Bigg[ \dfrac{1}{n} \sum_{i=1}^n \Big( -2g(\widehat{h}_{n,SP}, Z_i) + \E_{D_n'} \big[g(\widehat{h}_{n,SP}, Z_i')\big] -2\log f(Y_i - \widehat{h}_{n,SP}(X_i)) + 2\log f(Y_i - h^*(X_i))  \Big) \Bigg]  
\\
\nonumber & = \E_{D_n} \Bigg[ \dfrac{1}{n} \sum_{i=1}^n \Big(-2g(\widehat{h}_{n,SP}, Z_i) + \E_{D_n'} \big[g(\widehat{h}_{n,SP}, Z_i')\big] -2\log f(Y_i - \widehat{h}_{n,SP}(X_i)) + 2J_n(\widehat{h}_{n,SP}) - 2J_n(\widehat{h}_{n,SP}) 
\\
\nonumber & \hspace{2cm} + 2\log f(Y_i -\widetilde{ h}_{\mathcal{H}_{\sigma, n}}(X_i)) - 2J_n \big(\widetilde{ h}_{\mathcal{H}_{\sigma, n}} \big)  - 2\log f(Y_i -\widetilde{ h}_{\mathcal{H}_{\sigma, n}}(X_i))  + 2\log f(Y_i - h^*(X_i))
\\
\nonumber & \hspace{2cm} + 2J_n \big(\widetilde{ h}_{\mathcal{H}_{\sigma, n}} \big)\Big) \Bigg] 
\\
\nonumber & \leq  \E_{D_n} \Bigg[ \dfrac{1}{n} \sum_{i=1}^n \Big( -2g(\widehat{h}_{n,SP}, Z_i) + \E_{D_n'} \big[g(\widehat{h}_{n,SP}, Z_i')\big] -2J_n(\widehat{h}_{n,SP})  -2\log f(Y_i - \widehat{h}_{n,SP}(X_i)) + 2J_n(\widehat{h}_{n,SP}) 
\\
\nonumber & \hspace{2cm} + 2\log f(Y_i -\widetilde{ h}_{\mathcal{H}_{\sigma, n}}(X_i)) - 2J_n \big(\widetilde{ h}_{\mathcal{H}_{\sigma, n}} \big)  - 2\log f(Y_i -\widetilde{ h}_{\mathcal{H}_{\sigma, n}}(X_i))  + 2\log f(Y_i - h^*(X_i))
\\
\nonumber & \hspace{2cm} + 2J_n \big(\widetilde{ h}_{\mathcal{H}_{\sigma, n}} \big)\Big) \Bigg] 
\end{align*}
\begin{align}\label{bound_excess_risk_v1}
\nonumber & \leq  \E_{D_n} \Bigg[ \dfrac{1}{n} \sum_{i=1}^n \Big( -2g(\widehat{h}_{n,SP}, Z_i) + \E_{D_n'} \big[g(\widehat{h}_{n,SP}, Z_i')\big]  -2J_n(\widehat{h}_{n,SP}) -2\big[\big(\log f(Y_i - \widehat{h}_{n,SP}(X_i)) - J_n(\widehat{h}_{n,SP}) \big)
\\
\nonumber & \hspace{2cm} - \big(\log f(Y_i -\widetilde{ h}_{\mathcal{H}_{\sigma, n}}(X_i)) -J_n \big(\widetilde{ h}_{\mathcal{H}_{\sigma, n}} \big) \big) \big] +2\big(-\log f(Y_i -\widetilde{ h}_{\mathcal{H}_{\sigma, n}}(X_i)) +\log f(Y_i - h^*(X_i)) \big)
\\
\nonumber & \hspace{2cm} + 2J_n \big(\widetilde{ h}_{\mathcal{H}_{\sigma, n}} \big)\Big) \Bigg] \\
\nonumber &  \leq \E_{D_n} \Bigg[ \dfrac{1}{n} \sum_{i=1}^n \Big(-2g(\widehat{h}_{n,SP}, Z_i) + \E_{D_n'} \big[g(\widehat{h}_{n,SP}, Z_i')\big]  -2J_n(\widehat{h}_{n,SP}) \Big) \Bigg] \\
\nonumber & \hspace{1cm}-2\E_{D_n} \Bigg[\dfrac{1}{n} \sum_{i=1}^n \Big(\big(\log f(Y_i - \widehat{h}_{n,SP}(X_i)) - J_n(\widehat{h}_{n,SP}) \big) - \big(\log f(Y_i -\widetilde{ h}_{\mathcal{H}_{\sigma, n}}(X_i)) -J_n \big(\widetilde{ h}_{\mathcal{H}_{\sigma, n}} \big) \big)  \Big) \Bigg]\\
\nonumber & \hspace{1cm} + 2\E_{D_n} \Bigg[ \dfrac{1}{n} \sum_{i=1}^n \Big(-\log f(Y_i -\widetilde{ h}_{\mathcal{H}_{\sigma, n}}(X_i)) +\log f(Y_i - h^*(X_i)) \Big) + 2J_n \big(\widetilde{ h}_{\mathcal{H}_{\sigma, n}} \big)\Bigg] \\
\nonumber &  \leq \E_{D_n} \Bigg[ \dfrac{1}{n} \sum_{i=1}^n \Big(-2g(\widehat{h}_{n,SP}, Z_i) + \E_{D_n'} \big[g(\widehat{h}_{n,SP}, Z_i')\big]  -2J_n(\widehat{h}_{n,SP}) \Big)   \Bigg]
\\
\nonumber & \hspace{1cm} + 2\E_{D_n} \Bigg[ \dfrac{1}{n} \sum_{i=1}^n \Big(-\log f(Y_i -\widetilde{ h}_{\mathcal{H}_{\sigma, n}}(X_i)) +\log f(Y_i - h^*(X_i)) \Big) + 2J_n \big(\widetilde{ h}_{\mathcal{H}_{\sigma, n}} \big)\Bigg] 
\\
\nonumber &  \leq \E_{D_n} \Bigg[ \dfrac{1}{n} \sum_{i=1}^n \Big(-2g(\widehat{h}_{n,SP}, Z_i) + \E_{D_n'} \big[g(\widehat{h}_{n,SP}, Z_i')\big] \Big) -2J_n(\widehat{h}_{n,SP})\Bigg] + 2\big(R(\widetilde{h}_{\mathcal{H}_{\sigma, n}}) - R(h^*) +J_n \big(\widetilde{ h}_{\mathcal{H}_{\sigma, n}} \big) \big)
\\
 &  \leq \widetilde{A}_{1, n} + 2\big(R(\widetilde{h}_{\mathcal{H}_{\sigma, n}}) - R(h^*) +J_n \big(\widetilde{ h}_{\mathcal{H}_{\sigma, n}} \big)\big),
\end{align}
where $\widetilde{A}_{1, n} :=  \E_{D_n} \Bigg[ \dfrac{1}{n} \sum_{i=1}^n \Big(-2g(\widehat{h}_{n,SP}, Z_i) + \E_{D_n'} \big[g(\widehat{h}_{n,SP}, Z_i')\big] \Big)  -2J_n(\widehat{h}_{n,SP}) \Bigg]$.
The next part of the demonstration will consist of deriving an upper bound of $\widetilde{A}_{1, n}$.

\medskip

Let $0 <\beta_n < 1$. 
As in the proof of Theorem \ref{excess_risk_bound__expo_strong_mixing}, define,
\begin{align*} 
T_{\beta_n}f = \left\{
\begin{array}{ll}
  f   & \text{ if } f\ge \beta_n  \\
  \beta_n   &  \text{otherwise},
\end{array}
\right.
\end{align*}
and $h_{\beta_n}^*$ satisfying for all $x \in \mathcal{X}$,
\begin{equation*} 
 h_{\beta_n}^*(x) = \underset{h: \|h\|_{\infty} < F_n}{\argmin} \E\big[-\log T_{\beta_n} f(Y - h(X))|X=x\big].   
\end{equation*}
Moreover, set,
\[G(g, Z_i) = \E_{D_n'}[g(h, Z_0)] - 2g(h, Z_0) \]
and for all predictor $h: \mx \rightarrow \my$,
\[ G_{\beta_n}(h, Z_0) = \E_{D_n}[g_{\beta_n}(h, Z_0)] - 2g_{\beta_n}(h, Z_0) \text{ and } g_{\beta_n}(h, Z_0) = -\log T_{\beta_n} f(Y_0 - h(X_0)) + \log T_{\beta_n} f(Y_0 - h_{\beta_n}^*(X_0) ).
\]
\noindent
By going as in the \textbf{Step 1} in the proof of Theorem \ref{Proof_excess_risk_thm1} (see (\ref{equa_bound_func_Gbetanv2})), we get,
\begin{align}\label{equa_bound_func_G_by_Gbetanv2}
    \nonumber \E_{D_n}\Bigg[\dfrac{1}{n} \sum_{i=1}^n \Big(G(h, Z_i) - G_{\beta_n}(h, Z_i) \Big)\Bigg] & = \E_{D_n}\Bigg[\dfrac{1}{n} \sum_{i=1}^n \Big(\E_{D_n}[g(h, Z_i)] -\E_{D_n}[g_{\beta_n}(h, Z_i)] + 2\big(g_{\beta_n}(h, Z_i) - g(h, Z_i) \big) \Big) \Bigg] \\
    & \leq 3\dfrac{\log n}{n}.
\end{align}
It follows from (\ref{equa_bound_func_G_by_Gbetanv2}) that
\begin{align}\label{equa_bound_func_G_by_Gbetanv3}
 \nonumber   \E_{D_n}\Bigg[\dfrac{1}{n} \sum_{i=1}^n \Big( G(h, Z_i) -2J_n(h) \Big) \Bigg] & =  \E_{D_n}\Bigg[\dfrac{1}{n} \sum_{i=1}^n \Big(G(h, Z_i) - G_{\beta_n}(h, Z_i) + G_{\beta_n}(h, Z_i) -2J_n(h) \Big)\Bigg] 
 \\
& \leq \E_{D_n}\Bigg[\dfrac{1}{n} \sum_{i=1}^n G_{\beta_n}(h, Z_0) -2J_n(h) \Bigg] + 3\dfrac{\log n}{n}.
\end{align}
%
%
%
From the definition of
$G(g, Z_0)$, the term $\widetilde{A}_{1, n}$ given above can be expressed as:
\begin{align*} 
  \widetilde{A}_{1, n} = \E_{D_n} \Big[\dfrac{1}{n} \sum_{i=1}^n \Big( G(\widehat{h}_{n,SP}, Z_i) -2J_n(\widehat{h}_{n,SP}) \Big) \Big]. 
\end{align*}
In addition to (\ref{equa_bound_func_G_by_Gbetanv3}), we have
\begin{align}\label{equa_bound_func_G_by_Gbetanv4}
\widetilde{A}_{1, n} \leq \E_{D_n}\Bigg[\dfrac{1}{n} \sum_{i=1}^n G_{\beta_n}(\widehat{h}_{n,SP}, Z_0) -2J_n(\widehat{h}_{n,SP}) \Bigg] + 3\dfrac{\log n}{n}.
\end{align}
Therefore, it suffices to bound $\E_{D_n}\Big[\dfrac{1}{n} \sum_{i=1}^n G_{\beta_n}(\widehat{h}_{n,SP}, Z_i) -2J_n(\widehat{h}_{n,SP}) \Big]$. 
Let us recall that,
\begin{equation}\label{equa_bound_abs_gbetan}
   \big|g_{\beta_n}(h, Z_0) \big| \leq 2|\log \beta_n|, ~ \var \big(g_{\beta_n}(h, Z_0) \big) \leq 2|\log \beta_n| \E\big[g_{\beta_n}(h, Z_0)\big], \text{ and }     \Big|\E\big[g_{\beta_n}(h, Z_0)\big] - g_{\beta_n}(h, Z_0)\Big| \leq 4|\log \beta_n|,
\end{equation}
see (\ref{equa_bound_abs_gbetan_V0}), (\ref{equa_bound_var_gbetan_V0}) and (\ref{equa_bound_abs_moy_gbetan}).
Consider for all $\rho >0$, $j \in \N$, the set 
\begin{equation*}
 \mathcal{H}_{n, j, \rho} \coloneqq  \{ h \in \mathcal{H}_{\sigma, n}: 2^{j - 1} \ind_ {\{j \ne 0 \} } \rho  \leq J_n(h) \leq  2^{j} \rho \}.   
\end{equation*}
For $\rho > 0$, we have
\begin{align}\label{equa_bound_Gbetan_v1}
\nonumber &  P\Bigg\{\dfrac{1}{n} \sum_{i=1}^n  G_{\beta_n}(\widehat{h}_{n,SP}, Z_i) -2J_n(\widehat{h}_{n,SP}) > \rho \Bigg\} \\
\nonumber & =  P\Bigg\{ \E_{D_n}\big[g_{\beta_n}(\widehat{h}_{n,SP}, Z_1)\big] - \dfrac{2}{n} \sum_{i=1}^n g_{\beta_n}(\widehat{h}_{n,SP}, Z_i) -2J_n(\widehat{h}_{n,SP}) > \rho \Bigg\}
\\
 \nonumber & =  P\Bigg\{ 2\E_{D_n}\big[g_{\beta_n}(\widehat{h}_{n,SP}, Z_1)\big] - \dfrac{2}{n} \sum_{i=1}^n g_{\beta_n}(\widehat{h}_{n,SP}, Z_i) > \rho + 2J_n(\widehat{h}_{n,SP}) + \E_{D_n}\big[g_{\beta_n}(\widehat{h}_{n,SP}, Z_i)\big] \Bigg\}
 \\
 \nonumber & =  P\Bigg\{\E_{D_n}\big[g_{\beta_n}(\widehat{h}_{n,SP}, Z_1)\big] - \dfrac{1}{n} \sum_{i=1}^n g_{\beta_n}(\widehat{h}_{n,SP}, Z_i) > \dfrac{1}{2} \big(\rho + 2J_n(\widehat{h}_{n,SP}) + \E_{D_n}\big[g_{\beta_n}(\widehat{h}_{n,SP}, Z_i)\big] \big) \Bigg\}
  \\
\nonumber& \leq  P\Bigg\{\exists h\in \mathcal{H}_{\sigma, n}: \E_{D_n}\big[g_{\beta_n}(h, Z_1)\big] - \dfrac{1}{n} \sum_{i=1}^n g_{\beta_n}(h, Z_0) > \dfrac{1}{2} \big(\rho + 2J_n(h) + \E_{D_n}\big[g_{\beta_n}(h, Z_0)\big] \big) \Bigg\}
  \\
 \nonumber & \leq  P\Bigg\{\underset{h\in \mathcal{H}_{\sigma, n}}{\sup}\dfrac{\dfrac{1}{n} \sum_{i=1}^n \Big[\E_{D_n}\big[g_{\beta_n}(h, Z_0)\big] - g_{\beta_n}(h, Z_0) \Big]}{\rho + 2J_n(h) + \E_{D_n}\big[g_{\beta_n}(h, Z_1)\big] } > \dfrac{1}{2}  \Bigg\}
 \\
  \nonumber & \leq  \sum_{j=1}^{\infty} P\Bigg\{\underset{h\in \mathcal{H}_{n,j,\rho}}{\sup}\dfrac{\dfrac{1}{n} \sum_{i=1}^n \Big[\E_{D_n}\big[g_{\beta_n}(h, Z_0)\big] - g_{\beta_n}(h, Z_0) \Big]}{\rho + 2J_n(h) + \E_{D_n}\big[g_{\beta_n}(h, Z_1)\big] } > \dfrac{1}{2}  \Bigg\}
  \\
 & \leq  \sum_{j=1}^{\infty} P\Bigg\{\underset{h\in \mathcal{H}_{n,j,\rho}}{\sup}\dfrac{\dfrac{1}{n} \sum_{i=1}^n \Big[\E_{D_n}\big[g_{\beta_n}(h, Z_0)\big] - g_{\beta_n}(h, Z_0) \Big]}{2^j \rho+ \E_{D_n}\big[g_{\beta_n}(h, Z_1)\big] } > \dfrac{1}{2}  \Bigg\}.
\end{align}
Remark that, for $\rho > 0, j \ge 0$ and $h\in \mathcal{H}_{n,j,\rho}$, $2J_n(h) \ge 2^j \rho$. Since $\E_{D_n}\big[g_{\beta_n}(h, Z_1)\big] > 0$, in addition to (\ref{equa_bound_Gbetan_v1}), we have
\begin{align*} 
\nonumber   P\Bigg\{\dfrac{1}{n} \sum_{i=1}^n  G_{\beta_n}(\widehat{h}_{n,SP}, Z_i) -2J_n(\widehat{h}_{n,SP}) > \rho \Bigg\}
\nonumber & \leq  \sum_{j=1}^{\infty} P\Bigg\{\underset{h\in \mathcal{H}_{n,j,\rho}}{\sup}\dfrac{\dfrac{1}{n} \sum_{i=1}^n \Big[\E_{D_n}\big[g_{\beta_n}(h, Z_1)\big] - g_{\beta_n}(h, Z_i) \Big]}{\sqrt{2^j \rho + \E_{D_n}\big[g_{\beta_n}(h, Z_1)\big]} \sqrt{2^j \rho} } > \dfrac{1}{2}  \Bigg\}
\\
& \leq  \sum_{j=1}^{\infty} P\Bigg\{\underset{h\in \mathcal{H}_{n,j,\rho}}{\sup}\dfrac{\dfrac{1}{n} \sum_{i=1}^n \Big[\E_{D_n}\big[g_{\beta_n}(h, Z_1)\big] - g_{\beta_n}(h, Z_i) \Big]}{\sqrt{2^j \rho + \E_{D_n}\big[g_{\beta_n}(h, Z_1)\big]}  } > \dfrac{1}{2}\sqrt{2^j \rho}  \Bigg\}.
\end{align*}
Set
\[\mathcal{G}_{n,j,\rho} \coloneqq \{g_{\beta_n}(h, \cdot): \R^d\times \mathcal{Y} \to \R: h\in \mathcal{H}_{n,j,\rho}\} .\]
Let $j\in \N$. According to (\ref{equa_bound_abs_gbetan}) and the Bernstein-type inequality for exponential strong mixing processes, see Lemma 6.1 in \cite{kengne2025deep}, we obtain
\begin{align*} 
 \nonumber  P\Bigg\{\underset{h\in \mathcal{H}_{n,j,\rho}}{\sup}\dfrac{\dfrac{1}{n} \sum_{i=1}^n \Big[\E_{D_n}\big[g_{\beta_n}(h, Z_1\big] - g_{\beta_n}(h, Z_i) \Big]}{\sqrt{2^j \rho+ \E_{D_n}\big[g_{\beta_n}(h, Z_1)\big]}  } > \dfrac{1}{2}\sqrt{2^j \rho}  \Bigg\} 
 &\leq \mathcal{N} \Big(\mathcal{G}_ {n, j, \rho}, \dfrac{2^ j \rho}{8}, \|\cdot\|_{\infty} \Big) \exp \Bigg[- \dfrac{C\dfrac{2^ j \rho}{64} \big(n/(\log n)^2 \big)}{6|\log \beta_n| + 4|\log \beta_n|}\Bigg]   \\
& \leq \mathcal{N} \Big(\mathcal{G}_ {n, j, \rho}, \dfrac{2^ j \rho}{8}, \|\cdot\|_{\infty} \Big) \exp \Bigg[- \dfrac{C(2^ j \rho) \big(n/(\log n)^2 \big)}{640|\log \beta_n|}\Bigg],    
\end{align*}
for some constant $C > 0$.
Let $g_{\beta_n}(h_1, \cdot), g_{\beta_n}(h_2, \cdot) \in \mathcal{G}_{n,j,\rho}$, and $(x, y) \in \R^d \times \R$. 
Recall $0< \beta_n \leq 1$ $\forall~ n\in \N$. From (\ref{Proof_log_Tbeta_Lipschitz}) see Lemma \ref{Lemma1} above, we obtain 

\begin{equation*}
    |g_{\beta_n}(h_1(x), y) - g_{\beta_n}(h_2(x), y)| \leq \dfrac{\mathcal{K}_f}{\beta_n} |h_1(x) - h_2(x)|.
\end{equation*}
Set $\varepsilon := \dfrac{2^j \rho}{8}$.In \cite{ohn2022nonconvex}, we have  
\begin{equation*} 
\mathcal{N}(\varepsilon, \mathcal{G}_{n, j, \rho}, \| \cdot \|_\infty) \leq \mathcal{N} (\frac{\varepsilon}{\mathcal{K}_f/\beta_n}, \mathcal{H}_{n, j, \rho}, \| \cdot \|_\infty).
\end{equation*} 
One can easily see that,
\begin{equation*} 
 \mathcal{H}_{n, j, \rho}  \subset \left\{  h \in \mathcal{H}_{\sigma}(L_{n}, N_{n}, B_{n}, F, \dfrac{2^j \rho}{\lambda_n}): \| \theta(h) \|_{ \text{clip}, \tau_n} \leq \frac{2^j \rho}{\lambda_n}  \right\}. 
 \end{equation*}
Thus, we get from \cite{ohn2022nonconvex}, 
\begin{align*} 
\mathcal{N} \Big(\varepsilon, \mathcal{G}_{n, j, \rho}, \| \cdot \|_\infty\Big)  \nonumber & \leq \mathcal{N}\Bigg(\dfrac{\varepsilon}{\mathcal{K}_f/\beta_n}, \mathcal{H}_{n, j, \rho}, \| \cdot \|_\infty \Bigg)  \leq \mathcal{N}\Bigg(\frac{\varepsilon}{\mathcal{K}_f/\beta_n}, \mathcal{H}_{\sigma}(L_{n},N_{n}, B_{n}, F, \frac{2^j \rho}{\lambda_n}),  \| \cdot \|_\infty \Bigg)
\\
& \leq  \exp\Bigg[2 \frac{2^j \alpha}{\lambda_n}(L_n + 1) \log \Bigg(\frac{(L_n + 1)(N_n + 1)B_n}{\varepsilon/(\mathcal{K}_f/\beta_n) - \tau_n (L_n + 1)((N_n + 1) B_n)^{L_n +1}} \Bigg) \Bigg].
\end{align*} 
By going along similar lines as in the proof Theorem 4.1 in \cite{kengne2025deep}, we have
\begin{align}
\nonumber & P\Bigg\{\dfrac{1}{n} \sum_{i=1}^n  G_{\beta_n}(\widehat{h}_{n,SP}, Z_i) -2J_n(\widehat{h}_{n,SP}) > \rho \Bigg\}  
\\
\nonumber & \leq  \sum _{j = 1}^{\infty} \exp\Bigg[2 \frac{2^j \rho}{\lambda_n}(L_n + 1) \log \Big(\frac{\beta_n (L_n + 1)(N_n + 1)B_n}{(2^j\rho)/(8\mathcal{K}_f) - \tau_n \beta_n (L_n + 1)((N_n + 1) B_n)^{L_n +1}} \Big) - \dfrac{C(2^ j \rho) \big(n/(\log n)^2 \big)}{640|\log \beta_n|} \Bigg].
\end{align}
Recall $\beta_n = \dfrac{1}{n^p}$ for some constant $p \ge 1$ (from \textbf{(A5)}) and that,
\begin{equation}\label{equa_network_archi}
 L_n \asymp \log(n), N_n \lesssim n^{\nu_1}, B_n \lesssim n^{\nu_2}, \lambda_n \asymp \frac{\big( \log(n) \big)^{\nu_3}}{n}, \tau_n \leq \dfrac{1}{16 \mathcal{K}_{\ell}(L_n + 1)(n(N_n + 1)B_n)^{L_n + 1} },  
\end{equation}
for some $\nu_1, \nu_2 >0, \nu_3 > 5$.
For sufficiently large $n$, $\tau_n$ given in (\ref{equa_network_archi}) satisfy the condition \\
$\tau_n \leq \dfrac{1}{16 \mathcal{K}_{\ell} \beta_n(L_n + 1)((N_n + 1)B_n)^{L_n + 1} }$.

\medskip

\noindent
Using similar arguments as in the proof of Theorem 4.1 in \cite{kengne2025deep}, and considering \textbf{case 1:} $\rho > 1$; \textbf{case 2:} $1/n <\rho < 1$; \textbf{case 3:} $0< \rho < 1/n < 1$, we have for some $n_0 \in \N$ and for all $n > n_0$:
\begin{equation}\label{equa_bound_Gbetan_v2}
\E\Big[\dfrac{1}{n} \sum_{i=1}^n  G_{\beta_n}(\widehat{h}_{n,SP}, Z_i) -2J_n(\widehat{h}_{n,SP}) \Big] \leq \displaystyle \int_0^{\infty}  P\Bigg\{\dfrac{1}{n} \sum_{i=1}^n  G_{\beta_n}(\widehat{h}_{n,SP}, Z_i) -2J_n(\widehat{h}_{n,SP}) > \rho \Bigg\} d\rho \lesssim \dfrac{\big(\log n\big)^3}{n}.
\end{equation}
According to (\ref{equa_bound_func_G_by_Gbetanv4}),(\ref{equa_bound_Gbetan_v2}), we obtain

\begin{align}\label{def_A1n_v2}
   \widetilde{A}_{1, n} & \leq \E_{D_n}\Bigg[\dfrac{1}{n} \sum_{i=1}^n G_{\beta_n}(\widehat{h}_{n,SP}, Z_i) -2J_n(\widehat{h}_{n,SP}) \Bigg] + 3\dfrac{\log n}{n} \lesssim \dfrac{\log n}{n} + \dfrac{\big(\log n\big)^3}{n}  \lesssim   \dfrac{\big(\log n\big)^3}{n} .
\end{align}
Hence, according to (\ref{bound_excess_risk_v1}) and (\ref{def_A1n_v2}), we have
\begin{equation*} 
 \E[R(\widehat{h}_{n,SP}) - R(h^*) ]  \lesssim  \dfrac{\big(\log n\big)^3}{n} + 2\big(R(\widetilde{h}_{\mathcal{H}_{\sigma, n}}) - R(h^*) +J_n \big(\widetilde{ h}_{\mathcal{H}_{\sigma, n}} \big)\big).
\end{equation*}
This completes the proof of the theorem.
\qed

\subsection{Proof of Corollary \ref{equa_oracle_ineq_expo_strong_mixing} }
Let $L_n, N_n, B_n$ and $F_n = F$ satisfy the conditions in Corollary \ref{equa_oracle_ineq_expo_strong_mixing}.
Let  $h_{n}^{\dagger}, \widetilde{h}_{\mathcal{H}_{\sigma, n}} \in \mathcal{H}_{\sigma, n} \coloneqq \mathcal{H}_{\sigma}(L_n, N_n, B_n, F)$ such that  
\begin{equation*} 
R (h_{n}^{\dagger})  + J_n(h_{n}^{\dagger}) \leq \underset{h \in \mathcal{H}_{\sigma, n}} {\inf} \left[ R(h)  + J_n(h) \right] + \dfrac{1}{n} ~ \text{ and } ~ R (\widetilde{h}_{\mathcal{H}_{\sigma, n}})  + J_n(\widetilde{h}_{\mathcal{H}_{\sigma, n}}) = \underset{h \in \mathcal{H}_{\sigma, n}} {\inf} \left[ R(h)  + J_n(h) \right].
\end{equation*}
Hence, it holds that,
\begin{equation}\label{bound_approxi_err_v2}
  R \big(\widetilde{h}_{\mathcal{H}_{\sigma, n}} \big) - R(h^*) + J_n \big( \widetilde{h}_{\mathcal{H}_{\sigma, n}} \big)  \leq  R(h_{n}^{\dagger})  - R(h^*)+ J_n(h_{n}^{\dagger})   \leq \underset{h \in \mathcal{H}_{\sigma, n}} {\inf} \left[ R(h) - R(h^*)  + J_n(h) \right] + \dfrac{1}{n}.  
\end{equation}
In addition to Theorem \ref{excess_risk_bound__expo_strong_mixing_thm2}, there exists a positive constant $C_1 > 0$ and $n_0\in \N$ such that for all $n > n_0$, we have
\begin{align}\label{bound_approxi_error}
\nonumber \E[R(\widehat{h}_{n,SP}) - R(h^*) ] & \lesssim  \big(R(\widetilde{h}_{\mathcal{H}_{\sigma, n}}) - R(h^*) + J_n(\widetilde{h}_{\mathcal{H}_{\sigma, n}})\big) + \dfrac{\big(\log n\big)^3}{n}
\\
\nonumber & \leq C_1\Bigg(\big(R(\widetilde{h}_{\mathcal{H}_{\sigma, n}}) - R(h^*) + J_n(\widetilde{h}_{\mathcal{H}_{\sigma, n}})\big) + \dfrac{\big(\log n\big)^3}{n} \Bigg)
\\
 &  \leq C_1\Bigg(\underset{h \in \mathcal{H}_{\sigma, n}} {\inf} \Big[ R(h) - R(h^*)   + J_n(h) \Big] +  \dfrac{\big(\log n\big)^3}{n} \Bigg).
\end{align}
Which completes the proof  of Corollary \ref{equa_oracle_ineq_expo_strong_mixing}.
\qed

\medskip

\subsection{Proof of Corollary \ref{excess_risk_bound_class_of_Holder_function_V1} and Corollary \ref{excess_risk_bound_class_of_Holder_function_V2}}
\textbf{1.} Let $L_n, N_n, B_n, F >0$ satisfying the conditions in Corollary \ref{excess_risk_bound_class_of_Holder_function_V1} and $h^{*} \in  \mathcal{C}^{s}(\mathcal{X}, \mathcal{K}^{*})$ with $s, \mk^{*} >0$. 
 Set $S_n \asymp n^{\frac{d}{\kappa s + d} } \log n$, $\mathcal{H}_{\sigma, n} \coloneqq \mathcal{H}_{\sigma}(L_n, N_n, B_n, F_n, S_n)$ (see (\ref{DNNs_Constraint})) 
 and 
$ \widetilde{\mathcal{H}}_{\sigma, n} := \{h \in  \mathcal{H}_{\sigma,n}, ~ \|h - h^*\|_{\infty, \mx}  \leq \epsilon_n\} $.
By going along similar lines as in the proof of Corollary 4.3 in \cite{kengne2025deep} with $\epsilon_n = \dfrac{ 1}{ n^{\frac{s}{\kappa s + d}} }$ and by using the oracle inequality in Corollary \ref{equa_oracle_ineq_expo_strong_mixing}, we get
\[  \sup_{h^{*} \in \mathcal{C}^{s}(\mathcal{X}, \mathcal{K}^{*}) } \Big(  \E[R(\widehat{h}_{n,SP}) - R(h^{*})] \Big) \lesssim \dfrac{(\log n )^{\nu}}{n^{\frac{\kappa s}{\kappa s + d}} }, \]
for all $\nu > 5$, and Corollary \ref{excess_risk_bound_class_of_Holder_function_V1} follows.

\medskip

\noindent
\textbf{2.}
Let $L_n, N_n, B_n, F_n = F$ satisfying the conditions in Corollary \ref{excess_risk_bound_class_of_Holder_function_V2}.
Choose $S_n \asymp  n \phi_{n} \log n$ and set  $\mathcal{H}_{\sigma, n} \coloneqq \mathcal{H}_{\sigma}(L_n, N_n,, B_n, F_n, S_n)$.
Consider the class of composition structured functions $\mathcal{G}(q, \bold{d}, \bold{t}, \boldsymbol{\beta}, \mk^{*})$.
By using the relation $J_n(h) \leq \lambda_n S_n$, the oracle inequality in Corollary \ref{equa_oracle_ineq_expo_strong_mixing} and by going in the same way as in the proof of Corollary 4.5 in \cite{kengne2025deep}, we obtain,

\[  \underset{h^{*} \in \mathcal{G}(q, \bold{d}, \bold{t}, \boldsymbol{\beta}, \mk^{*})}{\sup} \E[\mathcal{E}_{Z_0}(\widehat{h}_{n,SP}) ] \lesssim  \big( \phi_n ^{\kappa/2} \lor \phi_n \big) \big(\log n \big) \big)^{\nu + 1}, \]
for all $\nu > 5$, and Corollary \ref{excess_risk_bound_class_of_Holder_function_V2} holds.
\qed

\end{document}